%% file: main.tex
\documentclass[11pt]{article}  

\input{macros}
\input{packages_only_arxiv}

\input{theorems_std.tex}
\usepackage{etoolbox}

\usepackage{thmtools,thm-restate}
\usepackage{cleveref}

\newcommand{\citep}[1]{\cite{#1}}

\newcommand{\KF}[0]{\textup{KF}}
\renewcommand{\llabel}[1]{\text{\tiny\redd{#1}}\label{#1}}

\newcommand{\const}[0]{K}
\newcommand{\Tmin}[0]{\const_1Ld_u\log \pf{Ld_u}\de}
\newcommand{\Ksg}[0]{K_{u}}

\newtoggle{neurips}
\togglefalse{neurips}

\newtoggle{alt}
\togglefalse{alt}

\newtoggle{arxiv}
\toggletrue{arxiv}

\begin{document}

\title{Improved rates for prediction and identification of\\ partially observed linear dynamical systems}

\author{Holden Lee\footnote{Duke University, Mathematics Department. \texttt{holden.lee@duke.edu}}}

\date{\today}
\maketitle


\begin{abstract}
Identification of a linear time-invariant dynamical system from partial observations is a fundamental problem in control theory. Particularly challenging are systems exhibiting long-term memory. A natural question is how learn such systems with non-asymptotic statistical rates depending on the inherent dimensionality (order) $d$ of the system, rather than on the possibly much larger memory length. We propose an algorithm that given a single trajectory of length $T$ with gaussian observation noise, learns the system with a near-optimal rate of $\widetilde O\left(\sqrt\frac{d}{T}\right)$ in $\mathcal{H}_2$ error, with only logarithmic, rather than polynomial dependence on memory length. We also give bounds under process noise and improved bounds for learning a realization of the system. Our algorithm is based on multi-scale low-rank approximation: SVD applied to Hankel matrices of geometrically increasing sizes. Our analysis relies on careful application of concentration bounds on the Fourier domain---we give sharper concentration bounds for sample covariance of correlated inputs and for $\mathcal H_\infty$ norm estimation, which may be of independent interest.
\end{abstract}

\input{intro}

\input{notation}

\input{our-results}

\input{prior-work}

\input{svd-approach}

\input{experiments}

\input{conclusion}

\section*{Acknowledgements}
The author thanks Xue Chen for useful discussions.

\printbibliography

\appendix
\input{linreg}

\input{learn-parameters}
\input{lower-bound}
\input{conc}

\input{experimental-details}

\end{document}

%% file: macros.tex
\newcommand{\C}[0]{\mathbb{C}}

\newcommand{\E}[0]{\mathbb{E}}

\newcommand{\sfF}[0]{\mathsf{F}}

\newcommand{\N}[0]{\mathbb{N}}

\newcommand{\Pj}[0]{\mathbb{P}}

\newcommand{\R}[0]{\mathbb{R}}
\newcommand{\bS}[0]{\mathbb{S}}
\newcommand{\sT}[0]{\mathsf{T}}
\newcommand{\T}[0]{\mathbb{T}}

\newcommand{\Z}[0]{\mathbb{Z}}

\newcommand{\one}[0]{\mathbbm{1}}

\newcommand{\de}[0]{\delta}
\newcommand{\De}[0]{\Delta}
\newcommand{\ep}[0]{\varepsilon}

\newcommand{\la}[0]{\lambda}

\newcommand{\rh}[0]{\rho}

\newcommand{\Te}[0]{\Theta}

\newcommand{\om}[0]{\omega}
\newcommand{\Om}[0]{\Omega}
\newcommand{\si}[0]{\sigma}
\newcommand{\Si}[0]{\Sigma}

\newcommand{\ot}[0]{\otimes}

\newcommand{\subeq}[0]{\subseteq}

\newcommand{\iy}[0]{\infty}

\newcommand{\rc}[1]{\frac{1}{#1}}
\newcommand{\prc}[1]{\pa{\rc{#1}}}

\newcommand{\fc}[2]{\frac{#1}{#2}}
\newcommand{\sfc}[2]{\sqrt{\frac{#1}{#2}}}
\newcommand{\pf}[2]{\pa{\frac{#1}{#2}}}

\newcommand{\ab}[1]{\left| {#1} \right|}
\newcommand{\an}[1]{\left\langle {#1}\right\rangle}
\newcommand{\ba}[1]{\left[ {#1} \right]}
\newcommand{\bc}[1]{\left\{ {#1} \right\}}

\newcommand{\ce}[1]{\left\lceil {#1}\right\rceil}

\newcommand{\pa}[1]{\left( {#1} \right)}

\newcommand{\ve}[1]{\left\Vert {#1}\right\Vert}

\newcommand{\ved}[0]{\ve{\cdot}}
\newcommand{\set}[2]{\left\{{#1}:{#2}\right\}}

\newcommand{\ub}[2]{\underbrace{#1}_{#2}}

\newcommand{\wt}[1]{\widetilde{#1}}

\newcommand{\wh}[1]{\widehat{#1}}

\newcommand{\amin}{\operatorname{argmin}}

\newcommand{\Hankel}{\operatorname{Hankel}}

\newcommand{\poly}{\operatorname{poly}}

\newcommand{\rank}{\operatorname{rank}}

\newcommand{\Supp}{\operatorname{Supp}}

\newcommand{\Toep}{\operatorname{Toep}}

\providecommand{\cal}[1]{\mathcal{#1}}
\renewcommand{\cal}[1]{\mathcal{#1}}

\newcommand{\pull}[9]{
#1\ar@/_/[ddr]_{#2} \ar@{.>}[rd]^{#3} \ar@/^/[rrd]^{#4} & &\\
& #5\ar[r]^{#6}\ar[d]^{#8} &#7\ar[d]^{#9} \\}

\newcommand{\cmp}[9]{
\xymatrix{
#1 \ar[r]^{#4}{#5} \ar@/_2pc/[rr]^{#8}_{#9} & #2 \ar[r]^{#6}_{#7} & #3
}
}

\newcommand{\ha}[1]{\ar@{^(->}[#1]}
\newcommand{\ls}[1]{\ar@{-}[#1]}
\newcommand{\sj}[1]{\ar@{->>}[#1]}
\newcommand{\aq}[1]{\ar@{=}[#1]}
\newcommand{\acir}[1]{\ar@{}[#1]|-{\textstyle{\circlearrowright}}}
\newcommand{\acil}[1]{\ar@{}[#1]|-{\textstyle{\circlearrowleft}}}
\newcommand{\ard}[1]{\ar@{.>}[#1]}
\newcommand{\mt}[1]{\ar@{|->}[#1]}
\newcommand{\inm}[1]{\ar@{}[#1]|-{\in}}
\newcommand{\inr}{\ar@{}[d]|-{\rotatebox[origin=c]{-90}{$\in$}}}
\newcommand{\inl}{\ar@{}[u]|-{\rotatebox[origin=c]{90}{$\in$}}}

\newcommand{\sumr}[2]{\sum_{\scriptsize \begin{array}{c}{#1}\\{#2}\end{array}}}

\newcommand{\sumo}[2]{\sum_{#1=1}^{#2}}
\newcommand{\sumz}[2]{\sum_{#1=0}^{#2}}

\newcommand{\coltwo}[2]{
\begin{pmatrix}
{#1}\\
{#2}
\end{pmatrix}}
\newcommand{\colthree}[3]{
\begin{pmatrix}
{#1}\\
{#2}\\
{#3}
\end{pmatrix}
}

\newcommand{\matt}[4]{
\begin{pmatrix}
{#1}&{#2}\\
{#3}&{#4}
\end{pmatrix}
}

\newcommand{\beq}[1]{\begin{equation}\llabel{#1}}
\newcommand{\eeq}[0]{\end{equation}}
\newcommand{\bal}[0]{\begin{align*}}
\newcommand{\eal}[0]{\end{align*}}
\newcommand{\ban}[0]{\begin{align}}
\newcommand{\ean}[0]{\end{align}}

\newcommand{\redd}[1]{{\color{red}#1}}

\newcommand{\llabel}[1]{\label{#1}\text{\fixme{\tiny#1}}}

\newcommand{\arxiv}[1]{\url{http://www.arxiv.org/abs/#1}}

\newcommand{\vocab}[1]{\textbf{#1}} 

\DeclareFontFamily{U}{wncy}{}
    \DeclareFontShape{U}{wncy}{m}{n}{<->wncyr10}{}
    \DeclareSymbolFont{mcy}{U}{wncy}{m}{n}
    \DeclareMathSymbol{\Sh}{\mathord}{mcy}{"58} 

%% file: packages_only_arxiv.tex
\usepackage{etex}

\usepackage{fullpage}

\usepackage{algorithm}
\usepackage{algpseudocode}
\algnewcommand\algorithmicinput{\textbf{Input: }}
\algnewcommand\INPUT{\State\algorithmicinput}
\algnewcommand\algorithmicinitialize{\textbf{Initialize: }}
\algnewcommand\INIT{\State\algorithmicinitialize}
\algnewcommand\algorithmicrun{\textbf{Run: }}
\algnewcommand\RUN{\State\algorithmicrun}
\algnewcommand\algorithmicupdate{\textbf{Update: }}
\algnewcommand\UPDATE{\State\algorithmicupdate}
\algnewcommand\algorithmicset{\textbf{Set: }}
\algnewcommand\SET{\State\algorithmicset}
\algnewcommand\algorithmicquery{\textbf{Query: }}
\algnewcommand\QUERY{\State\algorithmicquery}
\algnewcommand\algorithmicoutput{\textbf{Output: }}
\algnewcommand\OUTPUT{\State\algorithmicoutput}

\usepackage{amsmath}
\usepackage{amssymb}
\usepackage{amsthm}
\usepackage{array}
\usepackage{bbm}
\usepackage{cancel}
\usepackage{cmap}
\usepackage{enumerate}
\usepackage{enumitem}
\usepackage{fancyhdr}
\usepackage{mathdots}
\usepackage{mathtools}
\usepackage{mathrsfs}
\usepackage{hyperref}
\usepackage{stackrel}
\usepackage{stmaryrd}
\usepackage{tabularx}
\usepackage{tikz}
\usepackage{ctable}
\usepackage{titlesec}
\usepackage{titletoc}
\usepackage{url}
\usepackage{verbatim}
\usepackage{wasysym}
\usepackage{wrapfig}
\usepackage{yhmath}
\usepackage[all,cmtip]{xy}

\usepackage{hyperref}
\usepackage[%
	firstinits=true,
	doi=false,
	url=false,
	isbn=false,
	backend=bibtex,
	style=alphabetic,
	maxbibnames=99,hyperref]{biblatex}

\usetikzlibrary{calc,trees,positioning,arrows,chains,shapes.geometric,%
    decorations.pathreplacing,decorations.pathmorphing,shapes,%
    matrix,shapes.symbols,shadows,fadings}



%% file: theorems_std.tex
\newtheorem{thm}{Theorem}[section]
\newtheorem*{thm*}{Theorem}
\newtheorem{prb}[thm]{Problem}
\newtheorem*{prb*}{Problem}

\newtheorem*{ax*}{Axiom}

\newtheorem*{clm*}{Claim}

\newtheorem*{conj*}{Conjecture}
\newtheorem{cor}[thm]{Corollary}
\newtheorem{df}[thm]{Definition}
\newtheorem*{df*}{Definition}

\newtheorem*{ex*}{Example}

\newtheorem{lem}[thm]{Lemma}
\newtheorem*{lem*}{Lemma}

\newtheorem*{pos*}{Postulate}

\newtheorem*{pr*}{Proposition}

\newtheorem*{qu*}{Question}
\newtheorem{rem}{Remark}
\newtheorem*{rem*}{Remark}

%% file: intro.tex
\section{Introduction}

We consider the problem of prediction and identification of an \emph{unknown} partially-observed linear time-invariant (LTI) dynamical system with stochastic noise,
\begin{align}
\label{e:lds1}
x(t)&= A x(t-1) + B u(t-1)+ \xi(t) \\
\label{e:lds2}
y(t) & = C x(t) + D u(t) + \eta(t) ,
\end{align}
with a single trajectory of length $T$, given access only to input and output data. 
Here, $u(t)\in \R^{d_u}$ are inputs, $x(t)\in \R^d$ are the hidden states, $y(t)\in \R^{d_y}$ are observations (or outputs), 
$\xi(t)\sim N(0,\Si_x)$ and $\eta(t) \sim N(0,\Si_y)$ are iid gaussian noise,
and $A\in \R^{d\times d}, B\in \R^{d\times d_u} , C\in \R^{d_y\times d}, D\in \R^{d_y\times d_u}$ are matrices. 
Partial observability refers to the fact that we do not observe the state $x(t)$, but rather a noisy linear observation $y(t)$.

As a simple and tractable family of dynamical systems, LTI systems are a central object of study for control theory and time series analysis. The problem of prediction and filtering for a known system dates back to~\cite{kalman1960new}.
However, in many machine learning applications, the system is \emph{unknown} and must be learned from input and output data. Identification of an unknown system is often a necessary first step for robust control~\citep{dean2019sample,boczar2018finite}. In a long line of recent work, the interplay between machine learning and control theory has borne fruit 
in an improved understanding of the statistical and online learning guarantees for prediction, identification, and control for unknown systems.
In machine learning, LTI systems also serve as a simple model problem 
for learning from correlated data in stateful environments, 
and can give insight into understanding the successes of reinforcement learning~\citep{recht2019tour,tu2019gap} and recurrent neural networks~\citep{hardt2018gradient}.

Partial observability poses a significant challenge to system identification: In the fully observed setting, given access to $x(t)$, there is no obstacle to learning the matrices directly through linear regression. However, in the partially observed setting, the most natural form of the optimization problem is non-convex.

Systems exhibiting \emph{long-term memory} are particularly challenging to learn. Restricting to marginally stable systems, this occurs when the spectral radius of $A$, $\rh(A)$, is close to 1, and it implies that the output at a particular time cannot be accurately estimated without taking into account inputs over many previous time-steps---on the order of $O\prc{1-\rh(A)}$ times steps. Such systems often arise in practice. A particular class of such systems are those exhibiting \emph{multiscale} behavior, with different state variables that change on vastly different timescales~\citep{chatterjee2010dbns}. For example, the body's pH level is affected both by long-term changes on a timescale of days or weeks, as well as breathing rate which changes over a timescale of seconds. For such systems, it makes sense to discretize at the scale of the fastest changing variable, which leads to a long memory for the slowest-changing variable. With few exceptions, existing guarantees for learning partially observed LTI systems degrade as the memory length increases. However, counting the number of parameters in the model~\eqref{e:lds1}--\eqref{e:lds2} suggests that the right measure of statistical complexity is the intrinsic dimensionality of the system, not the memory length. This leads to the following natural question.

\textbf{Question:} How can we learn partially observed LTI systems with (non-asymptotic) statistical rates that depend on the \emph{intrinsic dimensionality} of the system, rather than the memory length?

Despite the simplicity of the question, little in the way of theoretical results are known. 
We focus on the particular problem of learning the \emph{impulse response (IR) function} of the system---which fully determines its input-output behavior---in $\cal H_2$ norm. This is a natural norm for prediction problems as it measures the expected prediction error under random input. Known guarantees for learning the IR depend on the memory length. One particularly undesirable consequence is that for a continuous system with time discretization $\Delta$ going to 0, the memory scales as $1/\Delta$ (while the system order stays constant), leading to suboptimal estimation by an arbitrarily large factor.


Our key contribution is an algorithm and analysis that gives statistical rates that are optimal up to logarithmic factors, in the absence of process noise. Unlike previous works, our rates depend on the system order $d$---the natural dimensionality of the problem---and only \emph{logarithmically} on the memory length of the system. 
Our algorithm is based on taking a low-rank approximation (SVD) of the Hankel matrix, which is a widely used technique in system identification. We consider a \emph{multiscale} version of this algorithm, where we repeat this process for a geometric sequence of sizes of the Hankel matrix. This is essential for obtaining a stronger theoretical guarantee.
In the setting of zero process noise, we prove that our algorithm achieves near-optimal $\wt O\pa{\sfc{d(d_u+d_y)}{T}}$ rates in $\cal H_2$ error for the learned system. 

Our analysis relies on careful application of concentration bounds on the Fourier domain to give sharper concentration bounds for sample covariance and $\cal H_\iy$ norm estimation, which may be of independent interest. While we consider our algorithm in a simple setting, we hope that this is a first step to understanding and improving more complex subspace identification algorithms.
Indeed, SVD and related spectral methods are a standard step used in subspace identification algorithms such as N4SID; our analysis suggests that SVD has an important ``de-noising effect''.

We also give improved bounds for system identification, that is, learning the matrices $A,B,C,D$ 
using the Ho-Kalman algorithm~\citep{ho1966effective}, with $\wt O\pa{\sfc{Ld(d_u+d_y)}{T}}$ rates.


%% file: notation.tex
\subsection{Notation}

\paragraph{Norms.} 
We use $\ved$ to denote the 2-norm of a vector. 
For a matrix $A$, let $\ve{A}=\ve{A}_2$ denote its operator norm, $\rh(A)$ denote its spectral radius (maximum absolute value of eigenvalue), and $\ve{A}_\sfF$ denote its Frobenius norm. For a matrix-valued function $M(t)\in \C^{d_1\times d_2}$, $\ve{M}_{\sfF}:= \sqrt{\sum_t \ve{M(t)}_\sfF^2}$. 
Let $\si_r(A)$ denote the $r$th singular value of $A$.

\paragraph{Fourier transform.}
Given a matrix-valued function $F:\Z\to\C^{m\times n}$, define the  (discrete-time) Fourier transform as the function $\wh F:\R/\Z \to \C^{m\times n}$ given by
\iftoggle{neurips}{$\wh F(\om) = 
    \sum_{t=-\iy}^{\iy}F(t) e^{-2\pi i \om t}$.}{
\begin{align*}
    \wh F(\om) &= 
    \sum_{t=-\iy}^{\iy}F(t) e^{-2\pi i \om t}.
\end{align*}}


\paragraph{Matrices.}
Given a sequence $(F(t))_{t=1}^{a+b-1}$ where each $F(t)\in \C^{m\times n}$,
define $\Hankel_{a\times b}(F)$ as the $am\times bn$ block matrix such that the $(i,j)$th block is $[\Hankel_{a\times b}(F)]_{ij} = F(i+j-1)$. 
Given a sequence $(F(t))_{t=0}^{a-1}$ where each $F(t)\in \C^{m\times n}$, define the Toeplitz matrix as the block matrix such that the $(i,j)$th block is $[\Toep_{a\times b}(F)]_{ij} = F(i-j)\one_{i\ge j}$. 
For a matrix $A$, let $A^\top, A^H, A^\dagger$ denote its transpose, Hermitian (conjugate transpose), and pseudoinverse, respectively. For a vector-valued function $v:\{a,\ldots, b\}\to \R^n$, let $v_{a:b}\in \R^{(|a-b|+1)n}$ denote the the vertical concatenation of $v(a),\ldots, v(b)$. Let $(A_1;\ldots;A_n)$ denote the vertical concantenation of matrices $A_1,\ldots, A_n$. 
Let $*$ denote convolution; we define convolutions between matrix and vector-valued functions by matrix-vector multiplication: $(F*u)(t) = \sum_{s\in \Z} F(s)u(t-s)$.

\paragraph{Control theory.}
For a matrix $A\in \C^{d\times d}$, define its resolvent as $\Phi_A(z) = (zI-A)^{-1}$.
For a linear dynamical system $\cal D$ given by~\eqref{e:lds1}--\eqref{e:lds2}, let $\Phi_{\cal D}=\Phi_{u\to y}$ denote the transfer function from $u$ to $y$ (response to input). Then
 \iftoggle{neurips}{ $\Phi_{\cal D} = \Phi_{u\to y}= C\Phi_A B + D = C(zI-A)^{-1}B + D$. }{
\begin{align*}
    \Phi_{\cal D}(z) = \Phi_{u\to y}(z) &= C\Phi_A(z) B + D = C(zI-A)^{-1}B + D 
\end{align*}}
Let $\T :=\set{z\in \C}{|z|=1}$ be the unit circle in the complex plane. For a matrix-valued function $F:\T\to \C^{d_1\times d_2}$, define the $\cal H_2$ and $\cal H_\iy$ norms by
\begin{align*}
    \ve{\Phi}_{\cal H_2} &=\sqrt{\rc{2\pi} \int_{\T} \ve{\Phi(z)}_F^2\,dz} &
    \ve{\Phi}_{\cal H_\iy} &= \sup_{z\in \T} \ve{\Phi(z)} .
\end{align*}
For a function $F:\N_0\to \C^{d_1\times d_2}$, define its Z-transform to be $\cal Z[F](z) =\sumz n\iy F(n) z^{-n}$. Considered as a function $\T\to \C$, we can take its $\cal H_2$ and $\cal H_\iy$ norms. Overloading notation, we will let $\ve{F}_{\cal H_p}:= \ve{\cal Z F}_{\cal H_p}$ for $p=2,\iy$.
The $\cal H_2$ and $\cal H_\iy$ norms can be interpreted as the Frobenius and operator norms of the linear operator from input to output, i.e., they measure the average power of the output signal under random or worst-case input, respectively. (Note however that there an implicit factor of $d$ scaling between $\mathcal H_2$ and $\mathcal H_\infty$.)
For background on control theory, see e.g.,~\cite{zhou1996robust}.



\paragraph{Constants.} In proofs, $\const$ may represent different constants from line to line.

%% file: our-results.tex
\section{Main results}
\label{s:results}
We consider the problem of prediction and identification for an unknown linear dynamical system~\eqref{e:lds1}--\eqref{e:lds2}. 
Our main goal is to obtain error guarantees in $\cal H_2$ norm, which determines prediction error under random input~\cite[Lemma 3.3]{oymak2018non}.



\begin{prb}\label{p:main}
Consider the partially-observed LTI system $\cal D$~\eqref{e:lds1}--\eqref{e:lds2} with gaussian inputs $u(t) \sim N(0,I_{d_u})$ for $0\le t<T$.
Suppose that the system is stable, that is, $\rh(A)<1$, and that we observe a single trajectory of length $T$ started with $x(0)=0$, that is, we observe $u(t)\sim N(0,I_{d_u})$ 
and $y(t)$ for $0\le t<T$. Suppose also that $D$ has rank at most $d$.

The goal is to learn a LTI system $\wt{\cal D}$ with the aim of minimizing $\ve{\Phi_{\wt{\cal D}} - \Phi_{\cal D}}_{\cal H_2}$.
Equivalently, letting 
\begin{align*}
    F^*(t) &= \begin{cases}
        D, & t=0\\
        CA^{t-1}B, & t \ge 1
    \end{cases}
\end{align*}
denote the impulse response function (also called the Markov parameters) of the system, the goal is to learn an impulse response $\wt F$ minimizing $\ve{F^* - \wt F}_{\cal H_2} = 
\ve{F^* - \wt F}_{\sfF}$.
\end{prb}
Note that learning $F^*$ is sufficient to fully understand the input-output behavior of the system, but we may also ask to recover the system parameters $A,B,C,D$ up to similarity transformation (see Theorem~\ref{t:param}). Note that we require $\rank (D)$ to be at most the system order so that it does not take more samples to learn than $A,B,C$.

Previous results~\citep{oymak2018non,sarkar2019finite} roughly depend polynomially on the ``memory'' $\rc{1-\rh(A)}$, which blows up as the spectral norm of $A$ approaches 1. 
In the setting of zero process noise, our goal is to obtain rates that are $\wt O\pf{\poly(d,d_u,d_y)}{\sqrt T}$, with only poly-logarithmic dependence on $\rc{1-\rh(A)}$. 
See Figure~\ref{fig} for a comparison. 

We assume that $\rh(A)<1$ because if $\cal D$ is not stable, it is in general impossible to learn $\wt{\cal D}$ with  finite $\cal H_2$ error, as a system with infinite response can have arbitrarily small response on any finite time interval. 
However, 
it may still be possible to learn the response up to time $L\ll T$ in this case~\citep{simchowitz2019learning}, or achieve other reasonable guarantees. The marginally stable case ($\rh(A)=1$) is an important case we leave to future work.

\begin{figure}[h!]
    \centering
\iftoggle{neurips}{
    \begin{tabular}{p{4cm}p{2cm}p{3cm}p{2cm}}
\hline 
Method & Rollout type & Min \# samples 
& IR error\tabularnewline
\hline 
\hline 
Least squares (IR) \cite{tu2017non} & Multi & $L$ & $\si\sfc{L}{T}$ \tabularnewline
\hline 
Least squares (IR) \cite{oymak2018non}
& Single & $L$ & $\si\sfc{L}{T}$ \tabularnewline
\hline 
Nuclear norm minimization 
& Multi & $\min\{d^2,L\}$ & $\si\sfc{L}{T}$ \tabularnewline
\cline{2-4}
\cite{sun2020finite} & Multi & $d$ & $\si\sfc{dL}{T}$ \tabularnewline
\hline 
rank-$d$ SVD (Theorem~\ref{t:svd-main}) & Single & $L$ & $\si\sfc{d}{T}$
\tabularnewline
\hline 
\end{tabular}
}{
    \begin{tabular}{|p{6cm}|p{3cm}|p{3cm}|p{2cm}|} 
\hline 
Method & Rollout type & Min \# samples 
& IR error\tabularnewline
\hline 
\hline 
Least squares (IR) \cite{tu2017non} & Multi & $L$ & $\si\sfc{L}{T}$ \tabularnewline
\hline 
Least squares (IR) \cite{oymak2018non}
& Single & $L$ & $\si\sfc{L}{T}$ \tabularnewline
\hline 
Nuclear norm minimization 
& Multi & $\min\{d^2,L\}$ & $\si\sfc{L}{T}$ \tabularnewline
\cline{2-4}
\cite{sun2020finite} & Multi & $d$ & $\si\sfc{dL}{T}$ \tabularnewline
\hline 
rank-$d$ SVD (Theorem~\ref{t:svd-main}) & Single & $L$ & $\si\sfc{d}{T}$
\tabularnewline
\hline 
\end{tabular}
}

    \caption{
    Here, $L$ is the memory length for the system, which is $\wt O\prc{1-\rh(A)}$ for well-conditioned systems.
    \emph{Rollout type} refers to whether we have access to a single trajectory or multiple trajectories. \emph{Min \# samples} refers to the minimum number of samples (up to log factors) before the bounds are operational. \emph{IR error} refers to the error in the impulse response in Frobenius/$\cal H_2$ norm.
    Logarithmic factors are omitted.}
    \label{fig}
\end{figure}


In our Algorithm~\ref{a:svd}, we first use linear regression to obtain a noisy estimate $F$ of the impulse response. Next, following standard system identification procedures, we form the Hankel matrix $\Hankel_{L\times L}(F)$ with the entries of $F$ on its diagonals. Because the true Hankel matrix 
\begin{align*}
\Hankel_{L\times L} (F^*) = \begin{pmatrix}CB & CAB & \cdots & CA^{L-1}B\\
CAB & CA^{2}B &  & \vdots\\
\vdots &  & \ddots & \vdots\\
CA^{L-1}B & \cdots & \cdots & CA^{2L-1}B
\end{pmatrix}
\end{align*}
has rank $d$, we take a low-rank SVD $R_L$ of the Hankel matrix to ``de-noise" the impulse response. 
We can then read off the estimated impulse response by averaging over the corresponding diagonal of $R_L$.  
For technical reasons, we need to repeat this process for a geometric sequence of sizes of the Hankel matrix: $L\times L$, $L/2\times L/2$, $L/4\times L/4$, and so forth. This is because the low-rank approximation objective for a $\ell \times \ell$ Hankel matrix $H$ encourages the skew-diagonals that are $\Te(\ell)$ 
(consisting of entries $H_{ij}$ with $\const_1\ell \le i+j \le c_2\ell$) to be close---as those are the diagonals with the most entries---and hence estimates $F^*(t)$ well when $t=\Te(\ell)$. In other words, low-rank estimation for $\Hankel_{\ell\times \ell}(F)$ is only sensitive to the portion of the signal that is at timescale $\ell$. Repeating this process ensures that we cover all scales.
 \iftoggle{alt}{\begin{algorithm2e}[h!]
 \DontPrintSemicolon
\caption{Learning impulse response through multi-scale low-rank Hankel SVD}
\KwIn{Length $L$ (power of 2), time $T$.}
Part 1: Linear regression to recover noisy impulse response\;
Let $u(t)\sim N(0,I_{d_u})$ for $0\le t<T$, and observe the outputs $y(t)\in \R^{d_y}$, $0\le t< T$.\;
Solve the least squares problem
\begin{align}\label{e:ls}
\min_{F:\Supp(F)\subeq [0,2L-1]} \sumz t{T-1} \ve{y(t)-(F*u)(t)}^2.
\end{align}
to obtain the noisy impulse response $F:[0,2L-1]\cap \Z \to \R^{d_y\times d_u}$.\;
Part 2: Low-rank Hankel SVD to de-noise impulse response\;
Let $\wt F(0) $ be the rank-$d$ SVD of $F(0)$\; 
\For{$k=0$ to $\log_2 L$}{
    Let $\ell = 2^k$. \;
    Let $R_{\ell}$ be the rank-$d$ SVD of $\Hankel_{\ell\times \ell}(F)$
    (i.e., $\amin_{\rank(R) \le d} \ve{R - \Hankel_{\ell\times \ell}(F)}$).\;
    For $\fc\ell2 <t\le \ell$, let $\wt F(t)$ be the $d_y\times d_u$ matrix given by $\wt F(t) = \rc t \sum_{i+j=t} (R_\ell)_{ij}$, where $(\cdot)_{ij}$ denotes the $(i,j)$th block of the matrix.\;
}
\KwOut{Estimate of impulse response $\wt F$.}
\label{a:svd}
\end{algorithm2e}}{}
\iftoggle{arxiv}{
\begin{algorithm}[h!]
\caption{Learning impulse response through multi-scale low-rank Hankel SVD}
\begin{algorithmic}[1]
\INPUT Length $L$ (power of 2), time $T$.
\State Part 1: Linear regression to recover noisy impulse response
\State Let $u(t)\sim N(0,I_{d_u})$ for $0\le t<T$, and observe the outputs $y(t)\in \R^{d_y}$, $0\le t< T$.
\State 
Solve the least squares problem
\begin{align}\label{e:ls}
\min_{F:\Supp(F)\subeq [0,2L-1]} \sumz t{T-1} \ve{y(t)-(F*u)(t)}^2.
\end{align}
to obtain the noisy impulse response $F:[0,2L-1]\cap \Z \to \R^{d_y\times d_u}$. 
\State Part 2: Low-rank Hankel SVD to de-noise impulse response
\State Let $\wt F(0) $ be the rank-$d$ SVD of $F(0)$.
\For{$k=0$ to $\log_2 L$}
    \State Let $\ell = 2^k$. 
    \State Let $R_{\ell}$ be the rank-$d$ SVD of $\Hankel_{\ell\times \ell}(F)$
    (i.e., $\amin_{\rank(R) \le d} \ve{R - \Hankel_{\ell\times \ell}(F)}$).
    \State For $\fc\ell2 <t\le \ell$, let $\wt F(t)$ be the $d_y\times d_u$ matrix given by $\wt F(t) = \rc t \sum_{i+j=t} (R_\ell)_{ij}$, where $(\cdot)_{ij}$ denotes the $(i,j)$th block of the matrix.
\EndFor 
\OUTPUT Estimate of impulse response $\wt F$.
\end{algorithmic}
\label{a:svd}
\end{algorithm}}{}

Our main theorem is the following.
\begin{thm} 
\label{t:svd-main}\label{t:main}
There is a constant $\const_1$ such that following holds. 
In the setting of Problem~\ref{p:main}, suppose that $F^*$ is the impulse response function, 
$G^*$ is the impulse response for the process noise ($G^*(t) = CA^{t}$, 
$T$ is such that 
$T\ge \Tmin$, 
$\ep_{\mathrm{trunc}}:=\ve{F^*\one_{[2L,\iy)}}_{\cal H_\iy}\sqrt{d_u} + \ve{G^*\one_{[2L,\iy)}}_{\cal H_\iy}\ve{\Si_x^{1/2}}_{\sfF} $, and
$M_{x\to y}= (O, C, CA, \ldots, CA^{L-1})^\top \in \R^{(L+1)d\times d_y}$.
Let $0<\de\le \rc 2$ and $\si = \sqrt{\ve{\Si_y} + \ve{\Si_x} L\log\pf{Ld_u}{\de}\ve{M_{x\to y}}^2}$. 
Then with probability at least $1-\de$, Algorithm~\ref{a:svd} 
learns an impulse response function $\wt F$ such that 
\begin{align*}
    \ve{\wt F - F^*}_{\sfF}
    &= 
    O\Bigg(\si \sfc{d\pa{d_y+d_u+\log\pf L\de}\log L}{T}+ \ep_{\mathrm{trunc}}\sqrt d+ \ve{F^*\one_{(L,\iy)}}_{\sfF}
    \Bigg).
\end{align*}
\end{thm} 
In the absence of process noise (when $\Si_x=O$), when $L$ and $T$ are chosen large enough, the first term dominates, and ignoring log factors, the dependence is $O\pa{\sfc{d(d_y+d_u)}{T}}$. 
We expect this to be the optimal sample complexity up to logarithmic factors. 
However, in the presence of process noise, there is an undesirable factor of $\sqrt L \ve{M_{x\to y}}$, which (for well-conditioned matrices) is expected to be $O\pf{1}{1-\rh(A)}$ or $O(L)$. We leave it an open problem to improve the guarantees in this setting. 

\begin{rem}The $L$-factor dependence on the process noise is unavoidable with the current algorithm: when the process noise has covariance $\Si_x=I$ and decays after $L$ steps, it can cause perturbations of size $O(\sqrt{L})$ compared to the noiseless system. Even in the case $d=1$, when the impulse response function is $ae^{-kt/L}$ for a known $k$, the noise will cause the estimate of $a$ to be off by $O(\sqrt{L})$, and hence the $\cal H_2$ norm of the impulse response to be off by $O(L)$.
Our algorithm only regresses on previous inputs, but in the presence of process noise, a better approach is to regress on both the previous inputs $u(t)$ and \emph{outputs} $y(t)$ and then take a (weighted) SVD, as in N4SID~\citep{qin2006overview}.
\end{rem}

\begin{rem}
Note that the burn-in time---the minimum trajectory length under which our error guarantees hold---is $\Om(L)$. 
A burn-in time of $\Omega(L)$ is information-theoretically required to get $\poly(d)$ rates. Attempting to extrapolate an impulse response function from time $o(L)$ to time $L$ can magnify errors by $\exp(d)$, because the finite impulse response of a system of order $d$ can approximate a polynomial of degree $d-1$ on $[0,L]$.
\end{rem}

\begin{rem}
Commonly, one assumes that $\ve{CA^tB}\le M\rho^{t}$ for some $M$ and $\rho$ (greater than or equal to the spectral radius of $A$). Then in the noiseless case, $\ep_{\text{trunc}}\le M\frac{\rho^{2L-1}}{1-\rho}\sqrt{d_u}$. The algorithm and theorem is stated known $\ep_{\text{trunc}}$ for simplicity. Knowing $\ep_{\text{trunc}}$ allows choosing an appropriate memory length $L$. 

A standard technique to convert the algorithm to one that achieves the correct rates as $T\to \infty$ without knowing the memory length is to use the ``doubling trick'': increase the memory length by a constant every time the number of timesteps doubles, so that the memory length scales as $\log_2(T)$. See e.g.,~\cite{tsiamis2020online}. This works because the impulse response decays exponentially. 
\end{rem}

We also show the following improved rates for learning the system matrices, by combining $\cal H_\iy$ bounds for the learned impulse response with stability results for the Ho-Kalman algorithm~\citep{oymak2018non}. Because the input-output behavior is unchanged under a similarity transformation $(A,B,C)\mapsfrom (W^{-1}AW,W^{-1}B, CW)$, we can only learn the parameters up to similarity transformation. We will make the standard control-theoretic assumptions that $\cal D$ is observable ($(C;CA;\ldots;CA^{d-1})$ is full-rank) and controllable ($(B,AB,\ldots,A^{d-1}B)$ is full-rank).
\begin{thm}\label{t:param}
Keep the assumptions and notation of Theorem~\ref{t:main}, suppose $\cal D$ is observable and controllable, and let 
\begin{align*}
    \ep' &= \si \sfc{L\pa{d_y+d_u+\log \pf L\de}}{T} + \ep_{\mathrm{trunc}}.
\end{align*}
Let $H^-=\Hankel_{L\times (L-1)}(F^*)$. Suppose that $\ep' =O( \si_{\min}(H^-))$. 
Then with probability at least $1-\de$, the Ho-Kalman algorithm 
(Algorithm~\ref{a:hk}) with $T_1=L, T_2=L-1$ applied to the least squares solution $F$ of~\eqref{e:ls}
returns $\wh A, \wh B, \wh C$ such that there exists a unitary matrix $W$ satisfying
\begin{align*}
    \max\bc{
    \ve{C-\wh C W}_\sfF,
    \ve{B-W^{-1}\wh B}_\sfF}
    &= O(\sqrt d\cdot \ep')\\
    \ve{A-W^{-1}\wh A W}_\sfF &=
    O\pa{\rc{\si_{\min}(H^-)}\cdot \sqrt d \cdot \ep' \cdot \pa{\fc{\ve{\Phi_{\cal D}}_{\cal H_\iy}}{\si_{\min}(H^-)}+1}}.
\end{align*}
\end{thm}
\iftoggle{alt}{\begin{algorithm2e}[h!]
 \DontPrintSemicolon
\caption{Ho-Kalman algorithm (from~\cite{oymak2018non})}
\KwIn{Length $T$, Markov parameter matrix estimate $F$, system order $d$, Hankel shape $(T_1,T_2+1)$ with $T_2+T_1+1=T$.}
Form the Hankel matrix $\wh H = \Hankel_{T_1,T_2+1}(F)$.\;
Let $\wh H^-\in \R^{d_y T_1 \times d_u T_2}$ be the first $d_uT_2$ columns of $\wh H$.\;
Let $\wh L = U\Si V^\top \in \R^{d_y T_1\times d_u T_2}$ be the rank-$n$ SVD of $\wh H^-$.\;
Let $\wh O\in \R^{d_yT_1\times d} = U\Si^{1/2}$.\;
Let $\wh Q\in \R^{d\times d_uT_2} = \Si^{1/2} V^\top$.\;
Let $\wh C$ be the first $m$ rows of $\wh O$.\;
Let $\wh B$ be the first $p$ columns of $\wh Q$.\;
Let $\wh H^+\in \R^{d_y T_1\times d_uT_2}$ be the last $d_uT_2$ columns of $\wh H$.\;
Let $\wh A = \wh O^\dagger \wh H^+ \wh Q^\dagger$.\;
\KwOut{$\wh A\in \R^{d\times d}, \wh B\in \R^{d\times d_u}, \wh C\in \R^{d_y\times d}$.}
\label{a:hk}
\end{algorithm2e}}{}
\iftoggle{arxiv}{\begin{algorithm}[h!]
\caption{Ho-Kalman algorithm (from~\cite{oymak2018non})}
\begin{algorithmic}[1]
\INPUT Length $T$, Markov parameter matrix estimate $F$, system order $d$, Hankel shape $(T_1,T_2+1)$ with $T_2+T_1+1=T$.
\State Form the Hankel matrix $\wh H = \Hankel_{T_1,T_2+1}(F)$.
\State Let $\wh H^-\in \R^{d_y T_1 \times d_u T_2}$ be the first $d_uT_2$ columns of $\wh H$.
\State Let $\wh L = U\Si V^\top \in \R^{d_y T_1\times d_u T_2}$ be the rank-$n$ SVD of $\wh H^-$.
\State Let $\wh O\in \R^{d_yT_1\times d} = U\Si^{1/2}$.
\State Let $\wh Q\in \R^{d\times d_uT_2} = \Si^{1/2} V^\top$.
\State Let $\wh C$ be the first $m$ rows of $\wh O$.
\State Let $\wh B$ be the first $p$ columns of $\wh Q$.
\State Let $\wh H^+\in \R^{d_y T_1\times d_uT_2}$ be the last $d_uT_2$ columns of $\wh H$.
\State Let $\wh A = \wh O^\dagger \wh H^+ \wh Q^\dagger$.
\OUTPUT $\wh A\in \R^{d\times d}, \wh B\in \R^{d\times d_u}, \wh C\in \R^{d_y\times d}$.
\end{algorithmic}\label{a:hk}\end{algorithm}}{}
As $L$ can be chosen to make $\ep_{\mathrm{trunc}}$ negligible, this gives $\wt O\pa{\sfc{Ld(d_u+d_y)}{T}}$ rates, however, with factors depending on the minimum eigenvalue of $H$. This is an improvement over the $\wt O\pa{\sqrt d\sqrt[4]{\fc{L(d_u+d_y)}{T}}}$ rates in~\cite{oymak2018non}. Note that our rates still a square-root dependence on the memory $L$; and leave it an open question whether one can obtain logarithmic dependence similar to Theorem~\ref{t:main}.

We prove Theorem~\ref{t:main} in Section~\ref{s:main-proof} and Theorem~\ref{t:param} in Appendix~\ref{s:param}. We give a lower bound in Section~\ref{a:lb} that shows that in the absence of process noise, the rate in Theorem~\ref{t:main} is optimal up to logarithmic factors.

%% file: prior-work.tex
\section{Related work} 
\label{s:prior-work}



We survey two classes of methods for learning partially observable LDS's, subspace identification and improper learning. With the exception of~\cite{rashidinejad2020slip}, 
all guarantees have sample complexity depending on the memory length $L$, which we wish to avoid. 

\subsection{Subspace identification}


The basic idea of subspace identification~\citep{ljung1998system,qin2006overview,van2012subspace} is to learn a certain structured matrix (such as a Hankel matrix), take a best rank-$k$ approximation (using SVD or another linear dimensionality reduction method), and learn the system matrices $A,B,C,D$ up to similarity transformation.
Usage of spectral methods circumvents the fact that the most natural optimization problem for $A,B,C,D$ is non-convex. However, classical guarantees for these methods are asymptotic. 

Recently, various authors have given non-asymptotic guarantees for system identification algorithms.  \cite{oymak2018non} analyzed the Ho-Kalman algorithm~\citep{ho1966effective} in this setting. \cite{sarkar2019finite} consider the setting where system order is unknown and give 
an end-to-end result for prediction, while \cite{tsiamis2020sample} consider the problem of online filtering, that is, recovering $x(t)$'s up to some linear transformation.
\cite{simchowitz2019learning} give guarantees under more general conditions of noise and marginal stability; however their main bound is for the \emph{truncated} Markov parameters, and to capture all but $\epsilon$ of the impulse response, we would have to truncate at the memory length $L$, which would incur dependence on $L$. Moreover, they require low phase rank, a condition which we do not expect to hold generically for ``random'' linear dynamical systems with eigenvalues close to 1. An advantage of their approach is that they are able to achieve consistent recovery of system parameters without taking the truncation length to be as large as the memory length.

An alternate, empirically successful approach is that of nuclear norm minimization or regularization~\citep{fazel2013hankel}. \cite{sun2020finite} (building on~\cite{cai2016robust}) give explicit rates of convergence, and show that the algorithm has a lower minimum sample complexity and is easier to tune. 

Our algorithm is based on the classical approach of taking a low-rank approximation of the Hankel matrix, but we repeat this process with Hankel matrices of sizes $L\times L $, $L/2\times L/2$, $L/4 \times L/4$, and so forth; this is key modification that allows us to obtain better statistical rates. Our analysis builds on the analyses given in~\cite{oymak2018non,sun2020finite}. As essential part of the analysis is analyzing linear regression for correlated inputs, where we extend the work of~\cite{djehiche2019finite} to MIMO (multiple input multiple output) systems, as explained below.

\subsubsection{Linear regression with correlated inputs}

An important step in obtaining non-asymptotic rates for system identification is analyzing linear regression for correlated inputs. The most challenging step is to lower-bound the sample covariance matrix of inputs to the linear regression. A lower bound, rather than a matrix concentration result, is sufficient~\citep{mendelson2014learning,simchowitz2018learning,matni2019tutorial}; however, a concentration result is obtainable in our setting.

\cite{tu2017non} give non-asymptotic bounds for learning the finite impulse response for a SISO (single input single output) system in $\ell^\iy$ Fourier norm; however, they require $L$ rollouts of size $O(L)$ and hence $\Om(L^2)$ timesteps. 
Addressing the more challenging single-rollout setting, 
\cite{oymak2018non} obtain bounds for a single rollout of $\wt \Om(L)$ timesteps, by using concentration bounds for random circulant matrices~\cite{krahmer2014suprema} to 
derive concentration inequalities for the covariance matrix. 
These concentration inequalities for the covariance matrix were improved (by logarithmic factors) by \cite{djehiche2019finite}. 
Although \cite{djehiche2019finite} give an analysis in the SISO case, as we show in Theorem~\ref{t:conc}, the results can be extended to the MIMO case with an $\ep$-net argument.

\subsection{Improper learning using autoregressive methods}

Instead of solving the statistical problem of identifying parameters, another line of work develops algorithms for regret minimization in online learning. 
The goal is simply to do well in predicting future observations, with small loss (regret) compared to the best predictor in hindsight; the learned predictor is allowed to be improper, that is, take a different functional form. 
In the stochastic case, this allows prediction almost as well as if the actual system parameters were known; however, the framework also allows for adversarial noise.

One popular strategy for improperly learning the system is to 
learn a linear autoregressive filter over previous inputs and observations, or ARMA model. Naturally, because we are optimizing over a larger hypothesis class, the statistical rates depend on $L$ rather than the system order $d$.


\cite[Theorem 4.7]{Ghai20} consider the problem of online prediction for a fully or partially observed LDS, and give a regret bound that 
depends polynomially on the memory length $L$. Their approach works even for marginally stable systems, that is, systems with $\rh(A)\le 1$.  
See also \cite{AnavaHMS13,hazan17learning,hazan2018spectral,Kozdoba2019OnLineLO,tsiamis2020online,rashidinejad2020slip} for previous work using autoregressive methods. 

Of particular interest to us is~\cite{rashidinejad2020slip}, 
which gives rates independent of spectral radius. Building on~\cite{hazan17learning}, they observe that it suffices to regress on previous inputs and outputs projected to a lower-dimensional space. Their algorithm works in the setting of process noise and competes with the Kalman filter, but only when $A-KC$ has real eigenvalues, where $K$ is the Kalman gain. 

%% file: svd-approach.tex
\section{Proof of main theorem}
\label{s:svd}\label{s:main-proof}

In this section, we prove Theorem~\ref{t:svd-main}. The proof hinges on the following lemma, which shows that if we observe a low-rank matrix plus noise, then taking a low-rank SVD can have a de-noising effect, producing a matrix that is closer to the true matrix.
\begin{lem}[De-noising effect of SVD] \label{l:svd}
There exists a constant $\const$ such that the following holds.
Suppose that $A\in \C^{m\times n}$ is a rank-$r$ matrix, $\wh A = A + E$, and $\wh A_r$ is the rank-$r$ SVD of $\wh A$. Then
\begin{align}\label{e:main}
    \ve{\wh A_r - A}_{\sfF} &\le \const\sqrt{r} \ve{E}.
\end{align}
\end{lem}

Compare this with the original error $\ve{\wh A - A}_{\sfF}=\ve{E}_{\sfF}$, which can only be bounded by $\sqrt{\min\{m,n\}}\ve{E}$. \iftoggle{neurips}{}{It is an interesting question what the best constant $\const$ is.}  
When applied to the $d$-SVD of the Hankel matrix, this gives a factor of $\sqrt d$ rather than $\sqrt L$ for the error.

\begin{proof}
We have 
\begin{align}
\label{e:svd-simp1}
    \ve{\hat A_r - A}_{\sfF} &\le \sqrt{2r} \ve{\hat A_r - A}_2\\
    \label{e:svd-simp2}
    &\le \sqrt{2r} \pa{\ve{\hat A_r - \hat A}_2 + \ve{\hat A - A}_2}\\
    \label{e:svd-simp3}
    &\le 2\sqrt{2r}\ve{E}
\end{align}
where~\eqref{e:svd-simp1} follows from $\hat A_r - A$ having rank at most $2r$,~\eqref{e:svd-simp2} follows from the triangle inequality, and~\eqref{e:svd-simp3} follows from Weyl's Theorem:  $\ve{\hat A_r - \hat A}_2 \le \si_{r+1}(\hat A)\le \si_{r+1}(A) + \ve{E}=\ve{E}$.
\end{proof}

To prove Theorem~\ref{t:svd-main}, we will need to obtain bounds for 
$F:\{0, 1,\ldots, 2L-1\}\to \R^{d_y\times d_u}$
learned from linear regression in $\cal H_\iy$ norm.  
The following is our main technical result. 


\begin{restatable}{lem}{linreg}
\label{l:noise-f-liy-mimo}\label{l:linreg-mimo}\label{l:noise-f-liy2}\label{l:linreg2}
There are $\const_1,\const_2$ such that the following hold.
Suppose $y=F^**u + G^**\xi + \eta$ where $u(t)\sim N(0,I_{d_u})$, $\xi(t) \sim N(0, \Si_x)$, $\eta(t) \sim N(0,\Si_y)$ for $0\le t<T$, and $\Supp(F^*), \Supp(G^*)\subeq [0,\iy)$. 
Let $F=\amin_{F\in \{0,\ldots, L\} \to \R^{d_y\times d_u}} \sumz t{T-1} |y(t) - (F*u)(t)|^2$, 
$M_{G^*}= (G^*(0),\ldots, G^*(L))^\top \in \R^{(L+1)d\times d_y}$, and
$\ep_{\mathrm{trunc}} = \ve{F^*\one_{[L+1,\iy)}}_{\cal H_\iy} \sqrt{d_u} +\allowbreak \ve{G^*\one_{[L+1,\iy)}}_{\cal H_\iy}\ve{\Si_x^{1/2}}_{\sfF}$. 
For $0<\de\le \rc2$, $T\ge\Tmin$, $1\le L'\le L$, $-1\le a<L-L'$,
\begin{align*}
    &\ve{(F-F^*)\one_{[a+1,a+L']}}_{\cal H_\iy} \\
    &
    \le 
    \const_2 \ba{
    \sfc 1T \pa{\sqrt{\ve{\Si_y} L' \pa{d_u+d_y + \log \pf{L'}{\de}}} + 
    \sqrt{ \ve{\Si_x} L'L d_u \log\pf{Ld_u}{\de}}\ve{M_{G^*}}}
    + \ep_{\mathrm{trunc}}
    }
\end{align*}
with probability at least $1-\de$.
\end{restatable}
In the case $\Si_x=O$, this roughly says that the error in the learned impulse response, $F-F^*$, over any interval of length $L'$, has all Fourier coefficients bounded in spectral norm by $\wt O\pa{\sfc{L'(d_u+d_y)}T}$, what we expect if the error from linear regression is uniformly distributed over all frequencies.

A complete proof is in Appendix~\ref{s:linreg}; we give a brief sketch.  First, because the errors are Gaussian, the error from linear regression, $F-F^*$, follows a Gaussian distribution. To bound its covariance, we lower-bound the smallest singular value of the sample covariance of the inputs (Lemma~\ref{l:simin2}, Appendix~\ref{s:lb-covar}). 
Here, the difficulty is that the inputs are \emph{correlated}---the input at time $t$ is $u_{t:t-L}$. Fortunately, the translation structure means it is close to a submatrix of an infinite block Toeplitz matrix, which becomes block diagonal in the Fourier domain. This ``decoupling'' allows us to show concentration. Compared to the SISO setting in~\cite{djehiche2019finite}, we require an extra $\ep$-net argument. Once we have a bound on the covariance, we can bound any $\ve{(\wh{F - F^*})(\omega)}$ by matrix concentration (Appendix~\ref{s:max-mimo}); to bound the $\cal H_\iy$ norm it suffices to bound this over a grid of $\omega$'s (Lemma~\ref{l:interp}).

Bounding the error in $\cal H_\iy$ norm of the impulse response allows us to bound the error in operator norm of the Hankel matrix, as the following lemma shows.
\begin{lem}\label{l:hankel-hiy}
For any $F:\Z\to \C^{m\times n}$, we have $\ve{\Hankel_{a\times b}(F)} \le \ve{F}_{\cal H_\iy}$. 
\end{lem}
\begin{proof}
Note that when $v:\Z\to \C^n$, $\Supp(v)\subeq [0,b-1]$, we have $(F*v)_{b:b+a-1} = \Hankel_{a\times b}(F) v_{b-1:0}$. Hence for any $v$ set to 0 outside of $[0,b-1]$, using Parseval's Theorem and the fact that the Fourier transform of a convolution is the product of the Fourier transforms, we have
\begin{align*}
    \ve{\Hankel_{a\times b}(F) v_{b-1:0}}_{2}
    &\le \ve{F*v}_2 = \ve{\wh F\wh v}_2 \le 
    \sup_{\om\in [0,1]}\ve{\wh F(\om)}_2
    \ve{\wh v}_2 
    = \ve{F}_{\cal H_\iy}\ve{v}_2 
\end{align*}
This shows that $ \ve{\Hankel_{a\times b}(F)}\le \ve{ F}_{\cal H_\iy}$.
\end{proof}
\iftoggle{neurips}{}{
\begin{rem}
In fact, something stronger is true: we can bound $\ve{\Hankel_{a\times b}(F)}$ by the discrete Fourier transform of $F$ over $\Z/(a+b-1)$~\cite[Theorem 3]{sun2020finite}. For consistency, we stick to using the Fourier transform over $\Z$; in light of Lemma~\ref{l:interp}, this only affects the result by constant factors.
\end{rem}}

Theorem~\ref{t:svd-main} will follow from the following bound after an application of the triangle inequality.
\begin{lem}\label{l:main}
There are $\const_1,\const_2$ such that the following holds for the setting of Problem~\ref{p:main}.
Suppose $L$ is a power of 2, and $T\ge \Tmin$. Let 
$\ve{F^*\one_{[L+1,\iy)}}_{\cal H_\iy} \sqrt{d_u}+ \ve{G^*\one_{[L+1,\iy)}}_{\cal H_\iy}\ve{\Si_x^{1/2}}_{\sfF}$ and 
$M_{x\to y} = (O, C, CA, \ldots, CA^{L-1})^\top\in \R^{(L+1)d\times d_y}$. 
Then with probability at least $ 1-\de$, the output $\wt F$ given by Algorithm~\ref{a:svd} satisfies 
{\small
\begin{align*}
    \ve{(\wt F-F^*)\one_{[1,L]}}_{\mathsf F}
    &\le
    \const_2
    \Bigg(\sfc{\ve{\Si_y}d\pa{d_y+d_u+\log\pf L\de}\log L}{T} 
    + 
   \sfc{ \ve{\Si_x} L dd_u \log\pf{Ld_u}{\de}}{T}\ve{M_{G^*}}
   + \ep_{\mathrm{trunc}}\sqrt d
    \Bigg)
\end{align*}}
\end{lem}
\begin{proof}
We are in the situation of Lemma~\ref{l:linreg2} with $G^*(t) = CA^{t-1}\one_{t\ge 1}$.
Let 
$\cal H_{\ell}  = \Hankel_{\ell \times \ell}(F)$ and 
$\cal H_\ell^* = \Hankel_{\ell \times \ell}(F^*)$.
Suppose $\ell\le L$ is even. Note that 
\begin{align*}
    \cal H_\ell &= \ub{\Hankel_{\ell\times \ell} (F^*)}{\cal H_\ell^*} + \Hankel_{\ell\times \ell}(F-F^*)
\end{align*}
where $\cal H_\ell^*=\Hankel_{\ell\times \ell} (F^*)$ is a rank-$d$ matrix, with error term is bounded by
\begin{align}\nonumber
    &\ve{\Hankel_{\ell\times \ell}(F-F^*)}
    \le \ve{(F-F^*_{\mathrm{trunc}})\one_{[1,2\ell-1]}}_{\cal H_\iy}
    & \text{by Lemma~\ref{l:hankel-hiy}}\\
    &<\const 
    \Bigg[
    \sfc 1T \Bigg(\sqrt{\ve{\Si_y} \ell \pa{d_u+d_y + \log \pf{L'}{\de}}}\nonumber\\
    &\quad + 
    \sqrt{ \ve{\Si_x} \ell L d_u \log\pf{Ld_u}{\de}}\ve{M_{G^*}}
    +\ep_{\mathrm{trunc}}
    \Bigg]
    &\text{by Lemma~\ref{l:linreg2}}
    \label{e:siso1}
\end{align}
with probability at least $1-\de$.
Let $R_\ell$ be the rank-$d$ SVD of $ \cal H_\ell$. Then by Lemma~\ref{l:svd}, 
\begin{align}\label{e:siso2}
    \ve{R_\ell - \cal H_\ell^*}_{\sfF} 
    &=O\pa{\sqrt d\ve{\Hankel_{\ell\times \ell}(F-F^*)}}.
\end{align}
Now letting
\iftoggle{neurips}{$\wt F(t) = \rc t \sum_{i+j=t} (R_\ell)_{ij}$ when $\fc \ell 2<t\le \ell$}{
\begin{align*}
    \wt F(t) &= \rc t \sum_{i+j=t} (R_\ell)_{ij} &\text{when }\fc \ell 2<t\le \ell,
\end{align*}}
we have (using the fact that the mean minimizes the sum of squared errors)
\begin{align*}
    \ve{R_\ell - \cal H_\ell^*}_{\sfF}^2 &\ge 
    \sum_{t=\fc \ell 2+1}^\ell \sum_{i+j=t} \ve{(R_\ell)_{ij} - F^*(t)}_{\sfF}^2\\
    &\ge \sum_{t=\fc \ell 2+1}^\ell \pa{t \cdot \ve{\wt F(t) - F^*(t)}_{\sfF}^2} 
    \ge \pa{\fc \ell 2+1} \sum_{t=\fc \ell2+1}^\ell \pa{\ve{\wt F(t) - F^*(t)}_{\sfF}^2}.
\end{align*}
Note that we only get a lower bound with a factor of $\ell$ if we restrict to $t$ that is $\Te(\ell)$, i.e., restrict to diagonals that have many entries. This is the reason we will have to repeat this process for multiple sizes.
Hence
\begin{align*}
    \ve{(\wt F-F^*)\one_{[\fc \ell 2+1,\ell]}}_{\sfF}^2 &\le \rc{\ell/2} \ve{R_\ell - \cal H_\ell^*}_{\sfF}^2.
\end{align*}
Together with~\eqref{e:siso2} and~\eqref{e:siso1} this gives with probability $\ge 1-\de$ that
{\small
\begin{align*}
    \ve{(\wt F-F^*)\one_{[\fc \ell 2+1,\ell]}}_{\sfF} &\le 
    \const
    \Bigg(\sfc{\ve{\Si_y}d\pa{d_y+d_u+\log\pf L\de}}{T}
    + 
        \sfc{ \ve{\Si_x} L d_u \log\pf{Ldd_u}{\de}}{T}\ve{M_{G^*}}
    +
    \fc{\ep_{\mathrm{trunc}}\sqrt d}{\sqrt \ell}
    \Bigg)
\end{align*}}
Replacing $\de$ by $\fc{\de}{\log_2L}$, using a union bound over powers of 2, and summing gives
{\small\begin{align*}
    \ve{(\wt F-F^*)\one_{[1,L]}}_{\sfF}& = 
    O
    \Bigg(
    \sfc{\ve{\Si_y}d\pa{d_y+d_u+\log\pf L\de}\log L}{T}
    + \sfc{ \ve{\Si_x} Ld d_u \log\pf{Ld_u}{\de}\log L }{T}\ve{M_{G^*}}+
    \ep_{\mathrm{trunc}}\sqrt d
    \Bigg).
\end{align*}}
\end{proof}

\begin{proof}[Proof of Theorem~\ref{t:svd-main}]
We have the bound in Lemma~\ref{l:main}, and also the same bound for $\ve{(\wt F-F^*)(0)}_{\sfF}$ after applying Lemma~\ref{l:linreg2} to $(F-F^*)\de_0$ and then applying Lemma~\ref{l:svd}.
Finally, note that $\ve{(\wt F - F^*)\one_{(L,\iy)}}_{\sfF} = \ve{F^*\one_{(L,\iy)}}_{\sfF}$ and use the triangle inequality.
\end{proof}

%% file: experiments.tex
\section{Experiments}

We compared three algorithms for learning the impulse response function: least-squares, and low-rank Hankel SVD with and without the multi-scale repetition. 
We include details of the experimental setup in Appendix~\ref{s:experimental-details}. Note that to reduce the number of scales, we consider use a slight modification of our Algorithm~\ref{a:svd} which triples the size at each iteration instead.

The plots show the error $\ve{F^*\one_{[1,L]}-F}_2$, where $F$ is the estimated impulse response on $[1,L]$, averaged over 10 randomly generated LDS's, as a function of the time $T$ elapsed. We consider systems of order $d=1, 3, 5, 10$, and memory lengths $L=27, 81$. 

Using SVD significantly reduces the error, supporting our theory which shows that SVD has a ``de-noising'' effect. 
Additionally,  multiscale SVD has better performance than naive SVD when $d$ is moderate, $L$ is large, and data is limited, but the performance is similar in a data-rich setting.

\begin{figure}[h!]
\centering
\includegraphics[width=0.4\linewidth]{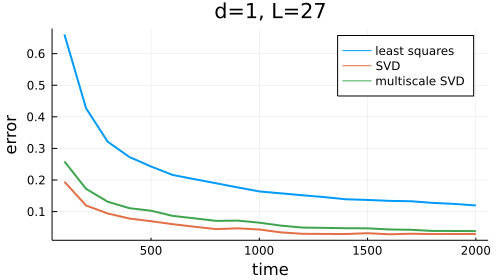} 
\includegraphics[width=0.4\linewidth]{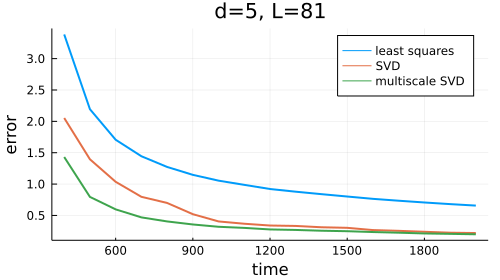} 
\end{figure}

%% file: conclusion.tex
\section{Open problems}
\label{s:open}


We conclude with some open problems. It would be interesting to obtain analogous rates (depending on system order) for the nuclear norm regularized problem~\citep{sun2020finite}. Spectral methods also suggest the possibility of obtaining regret bounds for adaptive control of partially-observed systems with milder dependence on $\rc{1-\rh(A)}$. We give several other problems below.

\paragraph{Process noise.} 
A natural open question is to obtain better guarantees in the presence of process noise $\xi(t)$.
Using the heuristic of parameter counting, we do not believe $\poly(L)$ memory dependence is necessary in this setting, and that the maximum likelihood estimator, although computationally inefficient, will attain memory-independent rates. Thus, the crux of the problem is to give a computationally efficient algorithm with mild memory dependence in this setting.

We note that in Theorem~\ref{t:svd-main}, the factor multiplying $\sqrt{\ve{\Si_x}}$ is $\sqrt L\ve{M_{x\to y}}_{2}$, which we expect to be on the order of $L$ when $L$ is the minimal sufficient memory length. This term arises because process noise can accumulate over $L$ timesteps.
In the case where $\xi(t)\sim N(0,\Si_x)$ is iid gaussian, the Kalman filter shows that we can rewrite the system in the predictor form~\citep{qin2006overview}
    \begin{align}
    x(t)^- &= A_{\KF} x(t-1)^- + B_\KF \coltwo{u(t-1)}{y(t-1)}\\
    y(t) &= Cx(t)^- + Du(t) + e(t)
\end{align}
where
$x(t)^-$ is the maximum likelihood and least squares estimator for $x(t)$ given the values of $u(s)$ and $y(s)$ for $s<t$; $A_{\KF}$ and $B_{\KF}$ are matrices which can be calculated in terms of $A, B, C, \Sigma_x, \Sigma_y$,  $e(t)\sim N(0,\Si_{\KF})$ for some covariance matrix $\Si_{\KF}$ that can be calculated in terms of $A,B,C,\Si_x,\Si_y$. 
This is now a \emph{filtering} problem, where we have to regress the output on both previous inputs $u(t)$ and outputs $y(t)$. 
This is more challenging, because unlike  previous $u(t)$, the previous $y(t)$ are highly correlated. 
One can perhaps treat this as a low-rank approximation in a different norm.

\paragraph{$\cal H_\iy$ error bounds.} How can we learn the system with $\cal H_\iy$ error bounds, that is, obtain error bounds under worst case input? This is particularly useful in control. We do not expect we can achieve $\sfc{d}{T}$ rates under iid inputs $u(t)$. However, it may be possible to take an active learning approach, by maximally exciting the system at frequencies we wish to learn, as in~\citep{wagenmaker2020active}.

\paragraph{Improved rates for learning system matrices.} Can we learning the system matrices with rates depending logarithmically on the memory $L$, perhaps by incorporating the multi-scale idea into system identification?

\paragraph{Input design.} In this work we choose iid random inputs, but can we estimate more efficiently with well-designed deterministic inputs? Can we design inputs to respect constraints such as constraints on frequencies? \cite{sarker2020parameter} suggests that efficient estimation is possible under general conditions on the inputs.

\paragraph{More general noise.} Do guarantees still hold if the noise satisfies weaker conditions such as sub-gaussianity? A key difficulty is bounding the maximum Fourier coefficient (as in Lemma~\ref{l:fft-iy-mimo}).

%% file: linreg.tex

\newpage

\section{Linear regression for impulse response}
\label{s:linreg}

In this section we prove Theorem~\ref{l:linreg-mimo}: under iid gaussian inputs, we can obtain high-probability error bounds for the transfer function of the learned impulse response in 
$\cal H_\iy$ norm. Moreover, these bounds kick in as soon as we have $\wt\Om(L)$ samples from a \emph{single} rollout. We note that analyzing the multiple-rollout setting as in~\cite{sun2020finite} is more straightforward, so we will not consider it here.

The main difficulty for analyzing linear regression is that the inputs are correlated. 
The most challenging step is to lower-bound the sample covariance matrix of inputs to the linear regression. 

In the SISO (single input single output) setting,~\cite{djehiche2019finite} give concentration bounds for the covariance matrix with $T=\wt \Om(L)$ timesteps. 
First, we extend this to the MIMO (multiple input multiple output) setting in Theorem~\ref{t:conc} (Note that~\cite{oymak2018non} consider the MIMO case but have extra log factors.)
Then, we use Gaussian suprema arguments as in~\cite{tu2017non} to obtain bounds for the transfer function in $\cal H_\iy$ norm (Lemma 
\ref{l:fft-iy-mimo}).

We suppose the inputs $u(0),\ldots, u(T-1)\sim N(0,I_{d_u})$ are iid, observe $y(0),\ldots, y(T-1)\in \R^{d_y}$, and perform linear regression on the finite impulse response $F:\{0,1,\ldots, L\} \to \R^{d_y\times d_u}$ (which we will also treat as an element of $\R^{(L+1)\times d_y\times d_u}$ without further comment). 

Recall that given a sequence $(F(t))_{t=0}^{a-1}$ where each $F(t)\in \C^{m\times n}$, the Toeplitz matrix is given by
\begin{align*}
    \Toep_{a\times b}(F) &= \begin{bmatrix}F(0) & 0 & \cdots & 0\\
F(1) & F(0) & \cdots & 0\\
\vdots & \vdots & \ddots & \vdots\\
F(a-1) & F(a-2) & \cdots & 
\end{bmatrix} \in \C^{am\times bn}.
\end{align*}

\paragraph{SISO setting.}
For simplicity, first consider the SISO setting: $d_u=d_y=1$ and $\eta(t)\sim N(0,1)$.
In this case, we learn a finite impulse response $f\in \R^{L+1}$ by minimizing the loss function 
\begin{align}\label{e:f-loss}
    \ve{y - u*f}_{[0,T-1]}^2  =
    \sumz t{T-1} \ve{y(t) - u_{t:t-L}^\top f}^2
    &=
    \ve{y_{0:T-1} - U f}^2
\end{align}
where we let 
$y_{0:T-1}$ denote the vertical concatenation of $y(0),\ldots ,y(T-1)$ and similarly for $u_{t:t-L}$, and let 
$U = \Toep_{T\times (L+1)}((u(t))_{t\ge 0})$. We set $u(t)=0$ for $t<0$. Solving the least-squares problem gives
\begin{align*}
    f &= (U^\top U)^{-1}U^\top y_{0:T-1}.
\end{align*}
Suppose that the data is generated as $y = f^**u + \eta$ where $\eta(t)\sim N(0,1)$ are independent and $f^*$ is supported on $[0,L]$. Later, we will consider the effect of truncating an infinite response. We abuse notation by considering $f,f^*$ both as functions $\Z\to \R$ and as vectors in $\R^{L+1}$, as they are supported in $[0,L]$. Similarly, we consider $y,\eta$ as vectors in $\R^T$. Then as vectors in $\R^T$, $y = Uf^* + \eta$. Hence the error is
\begin{align*}
    f - f^* &= (U^\top U)^{-1}U^\top (Uf^* + \eta) - f^*
    = (U^\top U)^{-1}U^\top  \eta
\end{align*}
Because $\eta$ has iid Gaussian entries, 
\begin{align*}
    f - f^* &\sim N(0, (U^\top U)^{-1}).
\end{align*}
To bound this, we need to bound $\ve{(U^\top U)^{-1}}$, and hence bound the smallest singular value of $U^\top U$. 

\paragraph{Notation for MIMO setting.} For a vector or matrix-valued function $F:\{a,a+1,\ldots, b\}\to \C^{d_1\times d_2}$, define 
\begin{align*}
M_{F,a:b} &= \colthree{F(a)^\top}{\vdots}{F(b)^{\top}}\in  \C^{(b-a+1)d_2\times d_1}
\end{align*}
with the indices omitted if they are clear from context.

\paragraph{MIMO setting.} In the general case, we would like to learn $F=(F(t)\in \R^{d_y\times d_u})_{t=0}^L\in \R^{(L+1)\times d_y\times d_u}$. 
Now suppose the data is generated as 
\begin{align*}
y &= F^* * u + G^* * \xi + \eta
\end{align*}
where $F^*, G^*$ are supported on $[0,\iy)$ and $\eta(t) \sim N(0, \Si_y)$, $\xi(t)\sim N(0,\Si_x)$, $t\ge 0$ are independent. Let $U=\Toep_{T\times (L+1)}((u(t)^\top)_{t=0}^{T-1})$ as before. Truncating $F^*$ and $G^*$, we have
\begin{align*}
y &= (F^*\one_{[0,L]} )* u + (G^*\one_{[0,L]}) * \xi + \eta + e\\
\text{where }
e(t) &= (F^* \one_{[L+1,\iy)}) * u + (G^*\one_{[L+1,\iy)}) * \xi.
\end{align*}
Thus, by taking the transpose and stacking vectors,
\begin{align*}
M_{y, 0:T-1} &= U M_{F^*, 0:L-1} + W M_{G^*, 0:L-1} + M_{\eta,0:T-1} + M_{e,0:T-1}\\
\text{where }
W&= \Toep_{T\times (L+1)}((\xi(t))^\top).
\end{align*}
The least squares solution $F$ minimizes $\ve{Y-UM_F}_\sfF^2$, so and the error is
\begin{align}\label{e:diffF}
M_F - M_{F^*} &= 
(U^\top U)^{-1} U^\top M_{\eta} + 
(U^\top U)^{-1} U^\top W M_{G^*} + (U^\top U)^{-1} M_e.
\end{align}


\subsection{Lower bounding sample covariance matrix}
\label{s:lb-covar}
In this subsection we lower bound the sample covariance matrix.

\begin{lem}\label{l:simin2}
There is a constant $\const$ such that the following holds. 
Let $u(t)\sim N(0,I_{d_u})$ and $U=\Toep_{T\times (L+1)}((u(t)^\top)_{t\ge 0})$. Then for $0< \de \le \rc 2$, $T\ge \Tmin$, 
\begin{align*}
\Pj\pa{\si_{\min}(U^\top U) \ge \fc T2} &\ge 1-\de.
\end{align*}
\end{lem}

This is a corollary of the following concentration bound, which generalizes Theorem 3.4 of~\cite{djehiche2019finite} to the MIMO setting. The main additional ingredient is an $\ep$-net argument to reduce to the analysis of the SISO case. We also swap out the chaining argument with a use of Lemma~\ref{l:interp}, which allows a shorter proof.
\begin{thm}\label{t:conc}
There is $\const$ such that the following holds. 
Suppose $u(t)$, $0\le t<T$ are independent, zero-mean, and $\Ksg$-sub-gaussian (see Definition~\ref{d:subg}), and let $U=\Toep_{T\times (L+1)}((u(t)^\top)_{t\ge 0})$. 
Then for $0<\de\le \rc 2$, $T\ge L$, 
\begin{align*}
    \ve{U^\top U - T I_{d_u}} 
    &\le \const \Ksg^2 
    \pa{
        Ld_u \log \pf{T}{\delta} + \sqrt{TLd_u \log \pf{T}{\delta}}
    }
\end{align*}
with probability $\ge 1-\de$.
\end{thm}
We first note the fact that infinite Toeplitz matrices become diagonal in the Fourier basis.
\begin{lem}\label{l:mp}
Consider the infinite block Toeplitz matrix $(Z(j-k))_{j,k\in \Z}\in \C^{(\Z\times d_1)\times (\Z\times d_2)}$, where $Z$ is a function $\Z\to \C^{d_1\times d_2}$. In the Fourier basis, it is given by the kernel $\wh Z(\om_1)\one_{\om_1=\om_2}$. That is, if $v:\Z \to \R^{d_2}$,  $\ve{Z}_1, \ve{v}_1<\iy$, then letting
\begin{align*}
    w(j) = \sum_k Z(j-k)v(k),
\end{align*}
we have
\begin{align*}
    \wh w(\om) = \wh Z(\om) \wh v(\om). 
\end{align*}
\end{lem}
Here, $\wh Z(\om)$ is called the \vocab{multiplication polynomial} of the matrix. 
\begin{proof}
Simply note that $w = Z*v$ and so $\wh w = \wh Z \wh v$.
\end{proof}

We will use the following lemma in order to bound the maximum of the Fourier transform by the maximum at a finite number of points.
\begin{lem}[\cite{bhaskar2013atomic}]
\label{l:interp}
Let $Q(z):= \sumz k{r-1} a_k z^{k}$, where $a_k\in \C$. 
For any $N\ge 4\pi r$, $\ve{Q}_{\cal H_\iy} 
\le \pa{1+\fc{4\pi r}{N}} \max_{j=0,\ldots, N-1} |Q(e^{\fc{2\pi i j}{N}})|$. 
\end{lem}

\begin{proof}[Proof of Theorem~\ref{t:conc}]
By rescaling we may suppose $\Ksg=1$.
Decompose 
\begin{align}
\label{e:U}
    U &= U_1 + U_2 \text{ where}\\
    \nonumber
    U_1 &= \begin{pmatrix}u(0)^{\top} &  & \boldsymbol{0}\\
\vdots & \ddots & u(0)^{\top}\\
u(T-L-1)^{\top} & \ddots & \vdots\\
\boldsymbol{0} &  & u(T-L-1)^{\top}
\end{pmatrix}\\
\nonumber
    U_2 &= \begin{pmatrix}\mathbf{0} & \cdots & \cdots & \mathbf 0\\
u(T-L)^{\top}&&& \vdots \\
\vdots & \ddots & & \vdots \\
u(T-1)^{\top} & \cdots & u(T-L)^{\top} & \mathbf 0
\end{pmatrix}.
\end{align}
Then 
\begin{align}\label{e:conc-decomp}
    U^\top U &= (T-L)I_{Ld_u} + (U_1^\top U_1 - (T-L) I_{Ld_u}) + U_1^\top U_2 + U_2^\top U_1 + U_2^\top U_2. 
\end{align}
Let $\mathsf T$ be the shift operator on functions: $\mathsf Tf(t) = f(t-1)$. 
Let $T' = T-L$ and let $u^{(1)} = u\one_{[0,T'-1]}$. 
Then the $(j,k)$th block of $U_1^\top U_1$ is
\begin{align*}
    (U_1^\top U_1)_{jk} &= 
    \sum_{t\in \Z}
    (\mathsf T^j u^{(1)})(t) (\mathsf T^k u^{(1)})(t)^\top
\end{align*}
Define the infinite block Toeplitz matrix in 
$\R^{(\Z\times d_u) \times (\Z\times d_u)}$ by $$Z_{jk} = \sumo tT (\sT^j u^{(1)})(t) (\sT^k u^{(1)})(t)^\top \one_{|j-k|\le L} - T' I_{\Z\times d_u}.$$ 
By Lemma~\ref{l:mp}, the multiplication polynomial of this matrix is
\begin{align}
\nonumber
    P_u(\om) &= \sum_{\ell=-L}^{L}\sum_{t\in \Z} (\sT^\ell u^{(1)})(t) u^{(1)}(t)^\top e^{-2\pi i \ell\om} - T'I_{d_u}\\
    \nonumber
    &=
    \sumr{j,k\in \Z}{|j-k|\le L} u(j) u(k)^\top e^{2\pi i (j-k)\om} - T'I_{d_u}\\
\label{e:conc1}
    &= \begin{pmatrix}u(0) & \cdots & u(T'-1)\end{pmatrix} M \colthree{u(0)^\top}{\cdots}{u(T'-1)^\top} - T'I_{d_u}
\end{align}
where $M\in \C^{\Z\times \Z}$ is the matrix with $M_{jk} = e^{2\pi i (j-k)\om}\one_{|j-k|\le L}$.
In  order to work with a scalar-valued function, we consider for $\ve{v}=1$
\begin{align*}
    v^\top P_u(\om) v &= \sum_{j,k\in \{0,\ldots, T'-1\}}
    \an{v,u(j)} \an{v,u(k)}
    e^{2\pi i \om(j-k)}\one_{|j-k|\le L} - T'\ve{v}^2.
\end{align*}
By Lemma~\ref{l:mp}, 
\begin{align*}
    \ve{U_1^\top U_1 - T' I_{Td_u}} 
    &\le \ve{Z - T' I_\Z} \le \ve{P_u(\om)}_{\cal H_\iy}.
\end{align*}
Taking $N=\ce{ 8\pi L}$ and noting $e^{2\pi i \om L}P(\om)$ is a polynomial of degree at most $2L$ in $e^{2\pi i \om}$, we have
\begin{align}
\nonumber
    \ve{P_u(\om)}_{\cal H_\iy}
    = \sup_{\om\in [0,1]}\ve{P_u(\om)}
    &= \sup_{\ve{v}=1}\sup_{\om\in [0,1]} |v^\top P_u(\om) v|\\
    \nonumber
    &= \sup_{\ve{v}=1} 2 \max_{\om\in \{0,\rc N,\cdots\}} |v^\top P_u(\om)v|
    &\text{by Lemma~\ref{l:interp}}\\
    \label{e:supp}
    &=2\max_{\om\in \{0,\rc N,\cdots\}} \pa{\sup_{v\in \cal N_\ep} |v^\top P_u(\om)v| + 3\ep \ve{P_u(\om)}
    }
\end{align}
where $\cal N_\ep$ is an $\ep$-net of the unit sphere in $\R^d$. (For arbitrary $v'$ with $\ve{v'}=1$, write $v=v+\De v$ where $v\in \cal N_\ep$ and $\ve{\De v}\le \ep$.) We first bound $v^\top P(\om)v$. Letting $w\in \R^{T'}$ be the vector with entries $w(j) = \an{v,u(j)}$, we have 
\begin{align*}
    v^\top P_u(\om) v &= w^\top M w - T'\ve{v}^2.
\end{align*}
Fix $v$. Because each $u(t)$ is independent $1$-subgaussian, each entry of $w$ is 1-subgaussian. By the Hanson-Wright inequality (Theorem~\ref{t:hw}), for some constant $c>0$, 
\begin{align*}
    \Pj(|v^\top P_u(\om) v |> s)
    &\le 2\exp\ba{-c\cdot \min\bc{\fc{s^2}{\ve{M}_{\sfF}^2}, \fc{s}{\ve{M}}}}.
\end{align*}
We calculate that $\ve{M}_{\sfF}^2 \le (2L+1)T$ and the Fourier transform of the function $e^{2\pi i \om t}\one_{|j-k|\le L}$ satisfies $\ve{\wh f}_\iy \le \ve{f}_1 \le 2L+1$, so by Lemma~\ref{l:mp}, $\ve{M}\le 2L+1$. Then for appropriate $\const$,
\begin{align*}
    \Pj\pa{|v^\top P_u(\om) v | > \const
    \pa{
    \sqrt{TL\log\prc{\de_1}}
    +L\log\prc{\de_1}
    }
    }
    \le \de_1.
\end{align*}
Next we bound $\ve{P_u(\om)}$ and choose $\ep$ appropriately. A crude bound with Markov's inequality suffices to bound $\ve{P_u(\om)}$. We have (because the second moment is at most the sub-gaussian constant)
\begin{align*}
    \E\ve{\begin{pmatrix}u(0) & \cdots & u(T'-1)\end{pmatrix}}_{\sfF}^2 
    &\le \E \sumo j{d_u} \sumz t{T'-1} \an{e_j,u(t)}^2 \le d_u T'
\end{align*}
so with probability $\ge 1-\de_2$, 
$\ve{\begin{pmatrix}u(0) & \cdots & u(T'-1)\end{pmatrix}}_{\sfF}^2 \le \fc{d_uT'}{\de_2}$. 
Hence, for every $\om\in [0,1]$, by~\eqref{e:conc1},
\begin{align*}
    \ve{P_u(\om)+T'I_{d_u}} &\le \ve{\begin{pmatrix}u(0) & \cdots & u(T'-1)\end{pmatrix}}_{\sfF}^2 \ve{M} 
    \le \fc{d_uT'}{\de_2} 2L.
\end{align*}
Choose $\ep = \fc{\de_2}{2d_uLT}$. Then with probability $\ge 1-\de_2$, we have 
\begin{align*}
\sup_{\om\in [0,1]} 3\ep\ve{P_u(\om)} \le 
3\cdot \fc{\de_2}{2d_uLT} \cdot \pa{\fc{d_uT'}{\de_2} 2L+T'}\le 
4.5.
\end{align*}
Now take $\de_1=\fc{\de}2$. By Cor. 4.2.13 of~\cite{vershynin2018high}, there is an $\ep$-net of size $|\cal N_\ep|\le \pa{1+\fc 2\ep}^{d_u} = \exp\pa{d_u\log\pa{1+\fc 2\ep}}=\exp\pa{d_u\log \pa{1+ \fc{8d_uLT}{\de}}}$. Letting $\de_1=\fc{\de}{2|\cal N_\ep|}$ and taking a union bound, with probability $1-\de$ we get 
\begin{align*}
    \eqref{e:supp}&\le \const\pa{
    \sqrt{TLd\log\pf T{\de}}
    +Ld\log\pf T{\de}
    }
\end{align*}

Next consider the term $U_1^\top U_2$.
Let $u^{(1)} = u\one_{[0,T'-1]}$, $u^{(2)} = u\one_{[T',T-1]}$. 
This is part of the infinite Toeplitz matrix  with $Z_{jk} =\sum_{t\in \Z} (\sT^j u^{(1)})(t) (T^k u^{(2)})(t)^\top \one_{|j-k|\le L-1}$. In the Fourier basis, 
\begin{align*}
    P_{u,12}(\om) &= 
    \sumr{j,k\in \Z}{|j-k|\le L} \sum_{t\in \Z}e^{-2\pi i j \om} (\sT^j u^{(1)})(t) (\sT^k u^{(2)})(t)^\top e^{2\pi i k\om}\\
    &\sumr{j,k\in \Z}{|j-k|\le L}
    \one_{[0, T'-1]}(j)\one_{[T-L,T-1]}(k)
    e^{2\pi i (j-k) \om}u(j)u(k)\\
    &=\begin{pmatrix}u(0) & \cdots & u(T'-1)\end{pmatrix} M \colthree{u(0)^\top}{\cdots}{u(T'-1)^\top}
\end{align*}
where $M_{jk} = \one_{[0, T'-1]}(j)\one_{[T-L,T-1]}(k) \one_{|j-k|\le L}e^{2\pi i (j-k) \om}$. 
As before, we have
\begin{align*}
    \ve{U_1^\top U_2} &\le 2 \max_{\om\in \{0,\rc N,\cdots\}} \sup_{v\in \cal N_\ep} |v^\top P_{u,12}(\om)v|+ 3 \ep \ve{P_{u,12}(\om)}.
\end{align*}
We calculate $\ve{M}_{\sfF}^2 \le(T-L)(2L+1)$ and each block in $M$ is part of a Toeplitz matrix, so similarly to before $\ve{M}\le 2L+1$. Hence, with probability at least $1-\de$,
\begin{align*}
    \ve{U_1^\top U_2} & \le 
    \const    \pa{
    \sqrt{TL\log\prc{\de}}
    +L\log\prc{\de}
    }
\end{align*}
Note $\ve{U_1^\top U_2} = \ve{U_2^\top U_1}$. Finally, we bound $U_2^\top U_2$.
Note $U_2$ is part of an infinite Hankel matrix with entries $u(T'+1)^\top,\ldots, u(T'+L)^\top$. The multiplication polynomial is 
\begin{align*}
    P_{u,2}(\om) &= 
    e^{-2\pi i (T-L)\om} \sumz t{L-1} u(T-L+t)^\top e^{-2\pi i t\om}.
\end{align*}
The real part is $\const\pa{\sum_{t=T-L}^{T-1} \cos^2 (-2\pi t\om)}^{1/2}$-sub-gaussian and the imaginary part is\\ $\const\pa{\sum_{t=T-L}^{T-1} \sin^2 (-2\pi t\om)}^{1/2}$-sub-gaussian for some constant $\const$. Hence 
\begin{align*}
    \Pj\pa{\an{e_j, P_{u,2}(\om)}^2 \le \const L \log\prc\de} &\ge 1-\de.
\end{align*}
Using this for $j=1,\ldots, d_u$, replacing $\de\mapsfrom \fc{\de}{d_u}$, and using a union bound gives
\begin{align*}
    \Pj\pa{\ve{P_{u,2}(\om)}^2 \le \const Ld_u \log\pf d\de} &\ge 1-\de.
\end{align*}
Now for $N\ge 4\pi L$, using another union bound gives
\begin{align*}
    \ve{U_2^\top U_2} 
    &\le \pa{\sup_{\om\in [0,1]}\ve{P_{u, 2}(\om)}}^2 
     \le \pa{2\max_{\om\in \{0, \rc N,\ldots\}}\ve{P_{u, 2}(\om)}}^2
    \le \const Ld_u \log \pf {d_uL}\de .
\end{align*}
with probability $\ge 1-\de$. Putting all the bounds together with~\eqref{e:conc-decomp} gives the theorem.
\end{proof}

\begin{proof}[Proof of Lemma \ref{l:simin2}]
For large enough $\const_2$, for $T\ge \const_2 L d_u \log\pf{T}{\de}$, we have that by Theorem~\ref{t:conc} that $\ve{U^\top U - TI_{d_u}}\le \fc T2$, so $\si_{\min}(U^\top U)\ge \fc T2$.

Finally, note that for large enough $\const_1$, $T\ge \const_1 L d_u \log \pf {Ld_u}{\de}$ implies $T\ge \const_2 L d_u \log\pf{T}{\de}$.
\end{proof}


We show here a bound similar to Theorem~\ref{t:conc} that will be useful to us later.
\begin{lem}\label{l:UTW}
There is a constant $\const$ such that the following holds. 
Suppose $u(t)$, $0\le t<T$ are independent, zero-mean, and $K_u$-sub-gaussian, and similarly for $w(t)$ with constant $K_w$. 
Let $U=\Toep_{T\times (L+1)}((u(t)^\top)_{t\ge 0})$, $W = \Toep_{T\times (L+1)}((w(t)^\top)_{t\ge 0})$. 
Then for $0< \de \le \rc 2$, $T\ge \Tmin$, 
\begin{align*}
    \ve{U^\top W} 
    &\le \const K_uK_w
    \pa{
        Ld_u \log \pf{T}{\delta} + \sqrt{TLd_u \log \pf{T}{\delta}}
    }
\end{align*}
with probability at least $1-\de$. 
\end{lem}
\begin{proof}
By scaling we may assume $K_u=K_w=1$. 
Decompose $U=U_1+U_2$ and $W=W_1+W_2$ as in~\eqref{e:U}. Let $S_a=\{0,\ldots, T-L-1\}$ and $S_b=\{T-L,\ldots, T-1\}$. We have
\begin{align*}
U^\top W &= \sum_{a,b\in \{0,1\}} U_a^\top W_b.
\end{align*}
Let $u^{(a)}=u\one_{S_a}$ and similarly define $w^{(b)}=w\one_{S_b}$. 
Then the $(j,k)$th block of $U_a^\top W_b$ is 
\begin{align*}
(U_a^\top W_b)_{jk} &= \sum_{t\in \Z} (\sT^j u^{(a)})(t) (\sT^k w^{(b)})(t)^\top .
\end{align*}
This is part of the infinite block Toeplitz matrix in $\R^{(\Z\times d_u)\times (\Z\times d_u)}$ defined by
\begin{align*}
Z_{jk} &= \sum_{t\in \Z} (\sT^j u^{(a)})(t) (\sT^k w^{(b)})(t)^\top \one_{|j-k|\le L} ,
\end{align*}
with multiplication polynomial
\begin{align*}
P_{u,ab}(\om) &= \sum_{|j-k|\le L} u^{(a)}(j) w^{(b)}(k) e^{2\pi i (j-k) \om}\\
&=  \begin{pmatrix}u(0) & \cdots & u(T-1)\end{pmatrix} M \colthree{w(0)^\top}{\cdots}{w(T-1)^\top}\\
\text{where } M_{jk} &= e^{2\pi i (j-k)} \one_{j\in S_a}\one_{k\in S_b} \one_{|j-k|\le L}.
\end{align*}
We calculate that $\ve{M}_F^2 \le T(2L+1)$ and $\ve{M}\le 2L+1$ so the same argument as in Theorem~\ref{t:conc} (but using the version of Hanson-Wright given by Corollary~\ref{c:hw}) gives that
\begin{align*}
\ve{U^\top W} &\le  \const K_uK_w
    \pa{
        Ld_u \log \pf{T}{\delta} + \sqrt{TLd_u \log \pf{T}{\delta}}
    }.
\end{align*}
\end{proof}

\subsection{Upper bound in $\cal H_\iy$ norm}
\label{s:max-mimo}

The following Lemma~\ref{l:fft-iy-mimo} generalizes the results of~\cite{tu2017non} to the MIMO setting.
To get the right dimension dependence, we will use the concentration bound for covariance given by Theorem~\ref{t:covar-conc}.


\begin{lem}\label{l:fft-iy-mimo}
There is a constant $\const$ such that the following holds.
Suppose that  $\eta(0),\ldots, \eta(T-1)\sim N(0,\Si)$ are iid, $\Phi\in \R^{(L+1)d_u\times T}$, and $E(0),\ldots, E(L)\in \R^{d_y\times d_u}$ are such that
\begin{align*}
\colthree{E(0)^\top}{\vdots}{E(L)^\top} = M_E = \Phi M_\eta\in \R^{(L+1)d_u\times d_y}
\end{align*}
For any $0<\de\le \rc 2$ and $-1\le a< L-L'$, 
\begin{align*}
    \Pj\pa{\ve{E \one_{[a+1,a+L']}}_{\cal H_\iy} \le {\const\sqrt{L'}\ve{\Si}^{\rc 2}\ve{\Phi}} \sqrt{d_y+d_u+\log\pf{L'}{\de}}} &\ge 1-\de.
\end{align*}
\end{lem}
\begin{proof}
First, by considering
\begin{align*}
M_E \Si^{-1/2} &= \Phi(M_\eta \Si^{-1/2}),
\end{align*}
we may reduce to the case where $\eta(t)\sim N(0,I_{d_u})$ are iid, i.e., all entries of $M_\eta$ are iid standard gaussian. 

Let $M_\om = (E\one_{[a+1,a+L']})^{\wedge}(\om)\in \C^{d_y\times d_u}$. 
Note that
\begin{align*}
M_\om &= (\phi_\om^H \ot I_{d_u})M_E\\
\text{where }\phi_\om &= (e^{2\pi i k \om}\one_{a+1\le k\le a+L'})_{0\le k\le L}\in \R^{L+1}
\end{align*}
as a column vector.  Because the columns of $M_\eta$ are independent and distributed as $N(0,I_T)$, the columns $m_j$ of $M_\om$ are independent. 
To bound $M_\om$, it suffices to bound $M_\om M_\om^H = \sumo j{d_y} m_j m_j^H$. 
Note that
\begin{align*}
    \ve{\E m_j m_j^H} &= \ve{(\phi_\om^H\ot I_{d_u})
    \Phi\Phi^\top (\phi_\om \ot I_{d_u})}\le L'\ve{\Phi}^2. 
\end{align*}
Let $\Phi'=(\phi_\om^H \ot I_{d_u})
    \Phi\Phi^\top (\phi_\om \ot I_{d_u})$. 
By Theorem~\ref{t:covar-conc}\footnote{The theorem is stated for real matrices, but we can view the matrix as acting on a real vector space of twice the dimension.}, 
\begin{align*}
    \Pj\pa{\ve{\rc{d_y} \sumo j{d_y} m_jm_j^H - \Phi'} \le \const L'\ve{\Phi'}^2\pa{\sfc{d_u+s}{d_y} + \fc{d_u+s}{d_y}}}&\ge 1-2e^{-s}\\
    \implies 
    \Pj\pa{\ve{\rc{d_y} \sumo j{d_y} m_jm_j^H} \le \const L'\ve{\Phi'}^2\pa{1 + \fc{d_u+\log\pf 2\de}{d_y}}}&\ge 1-\de
\end{align*}
by taking $u = \log\pf 2\de$. Multiplying by $d_y$ gives 
\begin{align*}
    \Pj\pa{
         \ve{M_\om M_\om^H} \le \const L'\ve{\Phi'}^2
        \pa{d_y+d_u + \log \pf 2\de}
    }&\ge 1-\de.
\end{align*}
Replacing $\de$ by $\fc{\de}{N}$, taking the square root, and taking a union bound gives
\begin{align}\label{e:fft-iy-mimo-1}
    \Pj\pa{\max_{\om \in \{0,\rc N,\ldots, \fc{N-1}N\}} \ve{M_\om} \le \const \sqrt{L'}\ve{\Phi'}^{\rc 2}
        \sqrt{d_y+d_u + \log \pf 2\de}} &\ge 1-\de.
\end{align}
Finally, we note that by Lemma~\ref{l:interp}, for $N=\ce{4\pi L'}$, 
\begin{align*}
    \ve{E\one_{[a+1,a+L']}}_{\cal H_\iy}
    &= \sup_{\om\in [0,1]} \ve{\wh{E\one_{[a+1,a+L']}}(\om)}
    = \sup_{\ve{v}_2=1} \sup_{\om \in [0,1]} 
    \ve{\wh{E\one_{[a+1,a+L']}}(\om)v}\\
    &\le \sup_{\ve{v}_2\le 1}
    \max_{\om \in \{0,\rc N,\ldots, \fc{N-1}N\}} 
    2\ve{\wh{E\one_{[a+1,a+L']}}(\om)v} \\
    &\le 
    2\max_{\om\in \{0,\rc N, \ldots, \fc{N-1}N\}} \ve{\wh{E\one_{[a+1,a+L']}}(\om)}.
\end{align*}
Combining this with~\eqref{e:fft-iy-mimo-1} gives the result.
\end{proof}

Finally we can put everything together to obtain a  $\cal H_\iy$ error bound for linear regression.



\linreg*

\begin{proof}
By~\eqref{e:diffF}, using the notation defined there, 
\begin{align*}
M_F - M_{F^*} &= 
\ub{(U^\top U)^{-1} U^\top M_{\eta}}{=:E_1} + 
\ub{(U^\top U)^{-1} U^\top W M_{G^*}}{=:E_2} + \ub{(U^\top U)^{-1} M_e}{=:E_3}.
\end{align*}
We wish to bound $\ve{(F-F^*)\one_{[a+1,a+L']}}_{\cal H_\iy} = \sup_{\om\in [0,1]}[(\phi_\om^H \ot I_{d_u})(M_F - M_{F^*})]$. 

We bound the contributions from $E_1,E_2,E_3$. First note that $\ve{(U^\top U)^{-1} U^\top}_2 = \ve{(U^\top U)^{-1}}_2^{1/2}$ and by Lemma~\ref{l:simin2}, for $T\ge \Tmin$, with probability $1-\de$, $\ve{(U^\top U)^{-1}}\le \fc 2T$. Call this event $\cal A$.
\begin{enumerate}
    \item 
    Under the event $\cal A$, by Lemma~\ref{l:fft-iy-mimo}, 
    \begin{align*}
        \Pj\pa{\sup_{\om\in [0,1]}[(\phi_\om^H \ot I_{d_u})E_1] \le  \const\sqrt{L'\ve{\Si_y}} \sfc 1T \sqrt{d_y+d_u+\log \pf {L'}\de}}\ge 1-\de.
    \end{align*}
    \item 
    By Lemma~\ref{l:UTW} and the condition on $T$,
    \begin{align*}
\ve{U^\top W} &\le  \ve{U^\top (W \Si_x^{-1/2})}\ve{\Si_x}^{1/2}
\le \sqrt{T Ld_u \log\pf{Ld_u}{\de}\ve{\Si_x}} 
\end{align*}
    Under $\cal A$, 
we bound the spectral norm (for all $\om$)
    \begin{align*}
\ve{(\phi_\om^H \ot I_{d_u})(U^\top U)^{-1}U^\top W M_{G^*}}
&\le 
\ve{\phi_\om^H \ot I_{d_u}}\ve{(U^\top U)^{-1}}\ve{U^\top W}\ve{ M_{G^*}}\\
&\le \sqrt{L'} \fc{\const}{\sqrt T} \sqrt{T Ld_u \log\pf{Ld_u}{\de}\ve{\Si_x}} \ve{M_{G^*}}\\
&\le \const\sqrt{\ve{\Si_x}L'Ld_u \log\pf{Ld_u}{\de}}\ve{M_{G^*}}.
\end{align*}
    \item 
    Let $\ep_{\mathrm{trunc},F} = \ve{F^*\one_{[L+1,\iy)}}_{\cal H_\iy}$ and similarly define $\ep_{\mathrm{trunc},G}$. 
    We bound the last term by noting
\begin{align*}
    \ve{(F^*\one_{[L+1,\iy)})*u}_2 
    &\le \ve{(F^*\one_{[L+1,\iy)})^{\wedge}\cdot \wh u}_2 
    \le \ep_{\mathrm{trunc},F}\ve{u}_2
\end{align*}
and similarly $\ve{(G^*\one_{[L+1,\iy)})*\xi}_2 \le  \ep_{\mathrm{trunc},G}\ve{\xi}_2$. 
We have $\Pj\pa{\ve{\eta_{0:T-1}} > \sqrt{Td_u}+ \const\sqrt{\log \prc\de}}\le \de$ 
and 
$\Pj\pa{\ve{\xi_{0:T-1}} > \sqrt{T}\ve{\Si_x^{1/2}}_{\sfF}+ \const\sqrt{\ve{\Si_x}\log \prc\de}}\le \de$
by Theorem~\ref{t:subg-conc}, so conditioned on event $\cal A$, 
\begin{align*}
\sup_{\om\in [0,1]}[(\phi_\om^H \ot I_{d_u})E_3]
&\le \const\sfc{1}T
\Bigg(\ep_{\mathrm{trunc},F} \pa{\sqrt{Td_u} + \sqrt{\log \prc\de}}
\\
&\quad +
\ep_{\mathrm{trunc},G} \pa{\sqrt{T}\ve{\Si_x^{1/2}}_\sfF + \sqrt{\ve{\Si_x}\log \prc\de}}
\Bigg)
\end{align*}
with probability at least $1-\de$. By the condition on $T$, the first terms are dominant.
\end{enumerate}
Finish by replacing $\de$ by $\fc \de4$ and using the triangle inequality and a union bound.
\end{proof}

%% file: learn-parameters.tex
\section{Improved rates for learning system matrices}
\label{s:sys-id}\label{s:param}
In this section, we combine Lemma~\ref{l:noise-f-liy2} and Lemma~\ref{l:hankel-hiy} with bounds in~\cite{oymak2018non} to give improved bounds for learning the system matrices.\footnote{References to~\cite{oymak2018non} are for the arXiv version \url{https://arxiv.org/abs/1806.05722}.}

As $L$ can be chosen to make $\ep_{\mathrm{trunc}}$ negligible, this gives $\sfc{Ld}{T}$ rates, however, with factors depending on the minimum eigenvalue of $H$.

We first re-do some of the bounds in~\cite{oymak2018non} more carefully, using their notation.
\begin{lem}[{\cite[Lemma B.1]{oymak2018non}}]
\label{l:hk-perturb1}
Suppose $\si_{\min}(L) \ge 2\ve{L-\wh L}$ where $\si_{\min}(L)$ is the smallest nonzero singular value of $L$. Let rank-$d$ matrices $L, \wh L$ have singular value decompositions $U\Si V^*$ and $\wh U \wh \Si \wh V^*$. There exists a $n\times n$ unitary matrix $W$ so that 
\begin{align*}
    \ve{U\Si^{1/2} - \wh U \wh \Si^{1/2} W}_F^2 + \ve{V\Si^{1/2} - \wh V \wh \Si^{1/2} W}_F^2  &\le 
     \fc{4(\sqrt 2+1)d\ve{L-\wh L}^2}{\si_{\min}(L)}.
\end{align*}
\end{lem}
\begin{proof}
This inequality is given as an intermediate inequality in the proof of Lemma B.1 in~\cite{oymak2018non}. The first line gives that the LHS is $\le \fc{2}{\sqrt2-1} \fc{\ve{L-\wh L}_F^2}{\si_{\min}(L)}$. Then use the fact that $\rank(L-\wh L)\le 2d$, so $\ve{L-\wh L}_F^2\le 2d \ve{L-\wh L}^2$.
\end{proof}

Using this instead of Lemma B.1 gives the following for Theorem 4.3 of~\cite{oymak2018non}.
\begin{lem}\label{l:hk-perturb2}
    Let $\wh A, \wh B, \wh C$ be the state-space realization corresponding to the output of Ho-Kalman with input $\wh G$. Suppose the system is observable and controllable. 
    Let $L=\Hankel_{L\times (L-1)}(F^*)$.
    Suppose $\si_{\min}(L)>0$ and  the low-rank approximation from Ho-Kalman  satisfies $\ve{L-\wh L}\le \si_{\min}(L)/2$. Then there exists a unitary matrix $W\in \R^{d\times d}$ such that 
    \begin{align*}
        \ve{B - W^{-1}\wh B}_F , \ve{C-\wh C W}_F &\le  \fc{2\sqrt{(\sqrt2+1)d}\ve{L-\wh L}}{\sqrt{\si_{\min}(L)}}\\
        \ve{A-W^{-1}\wh A W}_F &\le 
        \frac{2\sqrt d}{\si_{\min}(L)}
        \pa{
            \fc{2\sqrt{\sqrt2+1}\ve{L-\wh L}}{\si_{\min}(L)}\pa{2\ve{H^+} + \ve{H^+-\wh H^+}} + \ve{H^+-\wh H^+}
        }.
    \end{align*}
\end{lem}
\begin{proof}
We refer the reader to \cite{oymak2018non} for the details and just note the differences.
As in \cite{oymak2018non}, the first inequality follows from taking the square root in Lemma~\ref{l:hk-perturb1}.

For the second inequality, using Lemma~\ref{l:hk-perturb1}, the inequality for $\ve{O^\dagger - X^\dagger}_{\sfF}$ becomes instead 
\begin{align*}
    \ve{O^\dagger - X^\dagger}_{\sfF} &\le\ve{O-X}_{\sfF} \max\bc{\ve{X^\dagger}^2,\ve{O^\dagger}^2}\\
    &\le \fc{2\sqrt{(\sqrt 2+1)d}\ve{L-\wh L}_{\sfF}}{\si_{\min}(L)^{1/2}} \cdot \fc{2}{\si_{\min}(L)}\le \fc{4\sqrt{(\sqrt 2+1)d}\ve{L-\wh L}_F}{\si_{\min}(L)^{3/2}}
\end{align*}
so that (B.3)--(B.7) become
\begin{align*}
\ve{(O^\dagger - X^\dagger) H^+Q^\dagger}_{\sfF}
&\le \fc{4\sqrt{(2+\sqrt 2)d}\ve{L-\wh L}_{\sfF}}{\si_{\min}(L)^2}\ve{H^+}\\
\ve{X^{\dagger}\wh H^+ (Q^\dagger - Y^\dagger)}_{\sfF}
&\le \fc{4\sqrt{(2+\sqrt 2)d}\ve{L-\wh L}_{\sfF}}{\si_{\min}(L)^2}\pa{\ve{H^+}+\ve{H^+-\wh H^+}}.
\end{align*}
Substituting in (B.2) then gives the theorem.
\end{proof}
\begin{proof}[Proof of Theorem~\ref{t:param}]
Lemma~\ref{l:linreg2} gives a bound on $\ve{H-\wh H}$. 
By \cite[Appendix B.4]{oymak2018non},
\begin{align*}
    \ve{H^+-\wh H^+} &\le \ve{H-\wh H}, &
    \ve{H^+}&\le \ve{H}, &
    \ve{L-\wh L} &\le 2\ve{H-\wh H}.
\end{align*}
By Lemma~\ref{l:hankel-hiy}, $\ve{H}\le \ve{F^*}_{\cal H_\iy}= \ve{\Phi_{\cal D}}_{\cal H_\iy}$.
Plugging this into Lemma~\ref{l:hk-perturb2} gives the theorem.
\end{proof}

%% file: lower-bound.tex
\section{Lower bound}

\label{a:lb}

We prove a minimax lower bound which shows that the rates in Theorem~\ref{t:main} are optimal up to logarithmic factors. In fact, the lower bound already holds in the simple case where there is no time dependency, in which case the problem reduces to a low-rank linear regression problem.
\begin{thm}\label{t:lb}
Given the setting in Problem~\ref{p:main} with $\Si_\xi = \si^2 I_{d_y}$ and $\Si_\eta=O$, if $\max\{d_u,d_y\}\ge d$, then with 
probability $\ge 1-\de$ over the randomness of the $u(t)$, the minimax error for estimating $F^*$ (even when $A=O$, $B=O$, and $C=O$) in squared Frobenius norm is at least $\fc{\max\{d_u,d_y\}d\si^2}{T+ 2\pa{\sqrt{T\log\pf d\de} + \log\pf d\de }}$, i.e., any estimator $\wh F(y)$ satisfies
\begin{align*}
\max_{\cal D: A=O,B=O,C=O}
\ve{\wh F(y) - F^*}_{\sfF}^2 \ge 
\fc{\max\{d_u,d_y\}d\si^2}{T+ 2\pa{\sqrt{T\log\pf d\de} + \log\pf d\de }}.
\end{align*}
\end{thm}
Note alternatively we can restrict to $D=O$ and $A=O$; the only difference is that the impulse response is now nonzero only at $t=1$ rather than $t=0$.

We will use the following theorem. 
\begin{thm}[{\cite[Theorem 2.5]{candes2011tight}}]\label{t:rank-r-lb}
Suppose that $\cal A:\R^{n_1\times n_2}\to \R^m$ is a linear transformation such that 
\[\ve{\cal A(X)}_2^2 \le K \ve{X}_{\sfF}^2\]
for all matrices $X$ of rank at most $r$. 
Suppose that $n=\max\{n_1,n_2\}\ge r$. 
Given the observation $y=\cal A(X)+z$ where $z\sim N(0,\si^2I_m)$, the minimax error in squared Frobenius norm over $\set{X\in \R^{n_1\times n_2}}{\rank(X)\le r}$ is at least $\fc{nr\si^2}{K}$, i.e., any estimator $\wh M(y)$ satisfies
\[
\sup_{M:\rank(M)\le r} \E\ve{\wh M(y) - M}_{\sfF}^2 \ge \fc{nr\si^2}K.
\] 
\end{thm}
Note that \cite[Theorem 2.5]{candes2011tight} assumed a stronger condition---a matrix RIP (restricted isometry property)---but the lower bound coming from RIP is not used in the proof.

\begin{proof}[Proof of Theorem~\ref{t:lb}]
We are in the setting of Theorem~\ref{t:rank-r-lb} where $\cal A(D) = (Du(0);\ldots; Du(T-1))$. Note that by change of basis, we may assume $D$ is diagonal. If $\si_1,\ldots, \si_d$ are the singular values of $D$, then $\ve{\cal A(D)}_2^2= \sumz t{T-1}\ve{Du(t)}^2$ is distributed as $\sumo id \si_i^2 X_i$ where each $X_i\sim \chi_T^2$ is a $\chi^2$-random variable with $T$ degrees of freedom. By the tail bound in~\cite{laurent2000adaptive},
$\Pj(X_i\ge T+2(\sqrt{Tu}+u))\le e^{-u}$. Letting $u=\log \pf d\de$ and $Q=T+2(\sqrt{Tu}+u)$, we get that 
\begin{align*}
\Pj\pa{\max_{1\le i\le d} X_i \ge Q}\le d \Pj(X_1\ge Q)\le \de.
\end{align*}
Then with probability $\ge 1-\de$, the event $\max_{1\le i\le d} X_i < Q$ holds, so 
\[
\ve{\cal A(D)}_2^2 =\sumo id \si_i^2 X_i \le \pa{\sumo id \si_i^2}\max X_i \le \ve{D}_{\sfF}^2Q.
\]
Then by Theorem~\ref{t:rank-r-lb}, for any estimate $\wh D(y)$, $\sup_{D:\rank(D)\le d} \E\ve{\wh D(y) - D}_{\sfF}^2 \ge \fc{\max\{d_u,d_y\}d\si^2}Q$, as desired.
\end{proof}

%% file: conc.tex
\section{Concentration bounds}

In this section, we collect some useful concentration results.

\begin{df}\label{d:subg}
A $\R$-valued random variable $X$ is sub-gaussian with constant $K_X$ if $$\ve{X}_{\psi_2} := \inf\set{s>0}{\E[\exp((x/s)^2)-1]\le 1}\le K_X,$$ and a $\R^n$-valued random variable $X$ is sub-gaussian with constant $K_X$ if
$$
\ve{X}_{\psi_2}:= \sup_{v\in \bS^{n-1}}\ve{\an{X,x}}_{\psi_2}\le K_X.
$$
\end{df}

\begin{thm}[{\cite[Ex. 4.7.3]{vershynin2018high}}]
\label{t:covar-conc}
There is a constant $\const$ such that the following holds.
Let $X_1,\ldots, X_m$ be iid copies of a random vector $X$ in $\R^n$ satisfying the sub-gaussian bound for any $x$,
\begin{align*}
    \ve{\an{X,x}}_{\psi_2} &\le K_X \E[\an{X,x}^2].
\end{align*}
Let $\Si_m = \rc m \sumo im X_iX_i^\top$. Then for any $s\ge 0$,
\begin{align*}
    \Pj\pa{\ve{\Si_m - \Si} \le \const K_X^2\pa{\sfc{n+s}m + \fc{n+s}m}\ve{\Si}} &\ge 1-2e^{-s}.
\end{align*}
\end{thm}

\begin{thm}[{Hanson-Wright inequality, \cite[Theorem 1.1]{rudelson2013hanson}}]
\label{t:hw}
There is a constant $c>0$ such that the following holds. 
Let $A\in \C^{n\times n}$ be a matrix, and let $v\in \R^n$ be a random vector with independent, mean-0, $K_v$-sub-gaussian entries. Then for every $s\ge 0$, 
\begin{align*}
    \Pj(|v^\top A v - \E v^\top A v|> s)
    &\le 2\exp\ba{-c\cdot \min\bc{\fc{s^2}{K_v^4\ve{A}_{\sfF}^2}, \fc{s}{K_v^2\ve{A}}}}.
\end{align*}
\end{thm}

\begin{cor}\label{c:hw}
There is a constant $c>0$ such that the following holds. 
Let $A\in \C^{m\times n}$ be a matrix, and let $v\in \R^m,w\in \R^n$ be random vectors with independent, mean-0, $K_v$ and $K_w$ sub-gaussian entries, respectively. 
Then for every $s\ge 0$, 
\begin{align*}
    \Pj(|v^\top A w|> s)
    &\le 2\exp\ba{-c\cdot \min\bc{\fc{s^2}{K_v^2K_w^2\ve{A}_{\sfF}^2}, \fc{s}{K_vK_w\ve{A}}}}.
\end{align*}
\end{cor}
\begin{proof}
Apply Theorem~\ref{t:hw} for $v\mapsfrom \coltwo vw$ and $A\mapsfrom \matt OAOO$.
\end{proof}

\begin{thm}[{Sub-gaussian concentration, \cite[Theorem 2.1]{rudelson2013hanson}}]
\label{t:subg-conc}
There is a constant $c>0$ such that the following holds. 
Let $A\in \C^{m\times n}$ be a matrix, and let $v\in \R^n$ be a random vector with independent, mean-0, $K_v$-sub-gaussian entries. Then for every $s\ge 0$,
\begin{align*}
\Pj\ba{\ab{\ve{Av}_2 - \ve{A}_{\sfF}}>s} &\le 2\exp\pa{-\fc{cs^2}{K_v^4\ve{A}^2}}. 
\end{align*}
\end{thm}

%% file: experimental-details.tex
\section{Experimental details}

\label{s:experimental-details}

We generate random LDS's as follows. For $B$ and $C$, the rows or columns are chosen to be a random set of orthonormal vectors (depending on whether they have more rows or columns). 
For $A$, the entries are first chosen to be iid standard gaussians, and then $A$ is re-normalized so that its maximum eigenvalue has absolute value $\la_{\max}$. 
For simplicity, we take $D=O$.

We make a slight modification of Algorithm~\ref{a:svd} which triples the size at each iteration instead.
For $L=3^a$, we estimate a finite impulse response of length $4\cdot 3^{a-1}-1$. Then, for the multiscale SVD algorithm, at the $k$th scale ($k\ge 1$), we consider the rank-$d$ SVD of $\Hankel_{\ell\times \ell}(F)$, where $\ell = 2\cdot 3^{k-1}$, and use this SVD to estimate the $F(t)$ for $3^{k-1} < t \le 2^k$. For the single-scale SVD, we estimate all $F(t)$ from the rank-$d$ SVD of $\Hankel_{\ell\times \ell}(F)$, where $\ell = 2\cdot 3^{a-1}$.

The plots show the error $\ve{F^*\one_{[1,L]}-F}_2$, where $F$ is the estimated impulse response on $[1,L]$, averaged over 10 randomly generated LDS's, as a function of the time $T$ elapsed, for the following settings of parameters:
\begin{enumerate}
\item
$d=d_u=d_y=1$, $L=27$, $\la_{\max}=0.9$.
\item
$d=d_u=d_y=3$, $L=27$, $\la_{\max}=0.9$.
\item
$d=d_u=d_y=3$, $L=81$, $\la_{\max}=0.95$.
\item
$d=5$, $d_u=d_y=3$, $L=81$, $\la_{\max}=0.95$.
\item
$d=10$, $d_u=d_y=3$, $L=81$, $\la_{\max}=0.95$.
\end{enumerate}

The code was written in Julia, and is available at \url{https://github.com/holdenlee/hankel-svd}.

\begin{figure}
\includegraphics[width=0.5\linewidth]{lds-1-27.png} \includegraphics[width=0.5\linewidth]{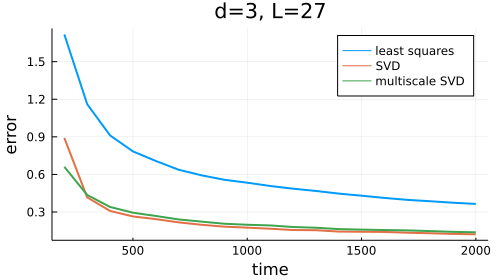} \\
\includegraphics[width=0.5\linewidth]{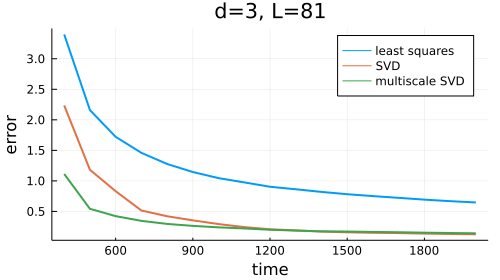}
\includegraphics[width=0.5\linewidth]{lds-5-81.png} \\
\includegraphics[width=0.5\linewidth]{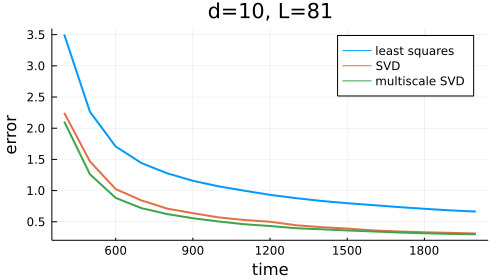}
\end{figure}